%% file: main.tex
\documentclass{article}

 \usepackage[preprint]{neurips_2026}

\usepackage[utf8]{inputenc} 
\usepackage[T1]{fontenc}    
\usepackage{hyperref}       
\usepackage{url}            
\usepackage{booktabs}       
\usepackage{amsfonts}       
\usepackage{nicefrac}       
\usepackage{microtype}      
\usepackage{xcolor}         
\usepackage{amsmath}
\usepackage{graphicx}
\usepackage{float}
\usepackage{multirow}
\usepackage{subcaption}
\usepackage{subfiles}
\usepackage{algorithm}
\usepackage{algpseudocode}
\usepackage{wrapfig}
\usepackage{placeins}

\usepackage{amsmath, amssymb, amsthm}

\newtheorem{proposition}{Proposition}

\newcommand{\ie}{\textit{i.e.},\ }
\newcommand{\eg}{\textit{e.g.},\ }
\newcommand{\etal}{\textit{et al.}\ }
\newcommand{\ksd}{score-KSD}

\usepackage{colortbl}

\theoremstyle{definition}

\title{Beyond Accuracy: Evaluating Posterior Fidelity of Diffusion Inverse Solvers}

\author{%
  Xiaoyu Qiu$^{1}$ \quad
  Taewon Yang$^{2}$ \quad
  Zhanhao Liu$^{2}$ \quad
  Guanyang Wang$^{3}$ \quad
  Liyue Shen$^{2}$ \\
  $^{1}$Department of Statistics, University of Michigan \\
  $^{2}$Department of EECS, University of Michigan \\
  $^{3}$Department of Statistics, Rutgers University \\
  \texttt{xiaoyuq@umich.edu} \quad
  \texttt{\{taewony, zhanhaol, liyues\}@umich.edu} \\
   \texttt{guanyang.wang@rutgers.edu}
}

\begin{document}

\maketitle

\begin{abstract}

Uncertainty evaluation is critical in scientific and engineering inverse problems. However, existing benchmarks on Diffusion Inverse Solvers (DIS) primarily focus on reconstruction accuracy but overlook uncertainty and distributional behavior. 
Since stochastic inverse solvers represent uncertainty through diffusion-based posterior samples, evaluating how well their generated samples capture the target posterior distribution becomes an important aspects of uncertainty quantification. 
To address this limitation and better understand this distributional behavior of diffusion samplers, we conduct a systematic study to investigate the posterior fidelity of a broad range of existing DIS methods in controlled simulation settings with known analytical true posterior.
Furthermore, to enable posterior-aware evaluation on real-world inverse problem where ground-truth posterior is unavailable, we propose score-based Kernel Stein Discrepancy (\textbf{\ksd}), a \textit{theoretically-grounded} and \textit{ground-truth-free} metric that measures the consistency of generated sample distribution from a DIS method with the target posterior score field, induced by the forward model and learned diffusion prior.  Through both simulation experiments and real-world inverse problem solving, we validate the effectiveness of proposed \ksd{} and demonstrate that it provides meaningful posterior fidelity diagnostics beyond reconstruction accuracy, revealing that \textit{higher reconstruction accuracy does not necessarily imply better posterior consistency}.

\end{abstract}

\subfile{S01Introduction}
\subfile{S02Background}

\subfile{S03-framework}

\subfile{S04ToySimulation}
\subfile{S05RealData}

\subfile{S06Discussion}

\section*{Acknowledgment}
Guanyang Wang acknowledges support from the National Science Foundation through grant
DMS–2210849 and an Adobe Data Science Research Award.
Liyue Shen acknowledges funding support by National Science Foundation (NSF) via grant IIS-2435746, Defense Advanced Research Projects Agency (DARPA) under contract No. HR00112520042, as well as the University of Michigan MIDAS PODS Grant Award.

\newpage

\bibliographystyle{plain}
\bibliography{References}

\appendix
\subfile{supp-ksd}
\subfile{supp-simulation}
\subfile{supp-exp}
\subfile{supp-proof}

\newpage

\end{document}

%% file: S01Introduction.tex
\section{Introduction}
Inverse problems are ubiquitous and fundamental across diverse scientific and engineering applications, including astronomy~\citep{inverse_astronomy_osti_5734250}, oceanography \citep{ocean_Wunsch_1996}, medical imaging \citep{song2022solvinginverseproblemsmedical,chung2022scoremri}, geophysics \citep{virieux_wave10.1190/1.3238367}, and audio signal processing \citep{audio_signal_lemercier2025diffusion,solve_audio_inverse_10095637}, among others. Recently, Diffusion Inverse Solvers (DIS) have emerged as a promising paradigm for solving these inverse problems, leveraging the generative power of pretrained diffusion models to regularize solutions effectively \cite{chung2022diffusionDPS,chen2025solvingDBP,cardoso2023monteMCGdiff,song2024Resamplesolvinginverseproblemslatent}.
Despite rapid algorithmic advancements, evaluation and benchmarking efforts lag behind, typically focusing on a set of natural image restoration tasks such as image denoising, deblurring and super-resolution~\citep{evalNEURIPS2021_6e289439,song2023pseudoinverseReddiff2,mardani2023reddiff1}. 
Furthermore, to evaluate real-world scientific applications with greater structural challenges in forward modeling, where priors and observations are governed by underlying physics, Zheng~\etal introduced InverseBench~\cite{zheng2025inversebench} , a comprehensive evaluation of existing diffusion inverse solver methods focused on scientific tasks.

\begin{figure}
    \centering
    \includegraphics[width=1.0\linewidth]{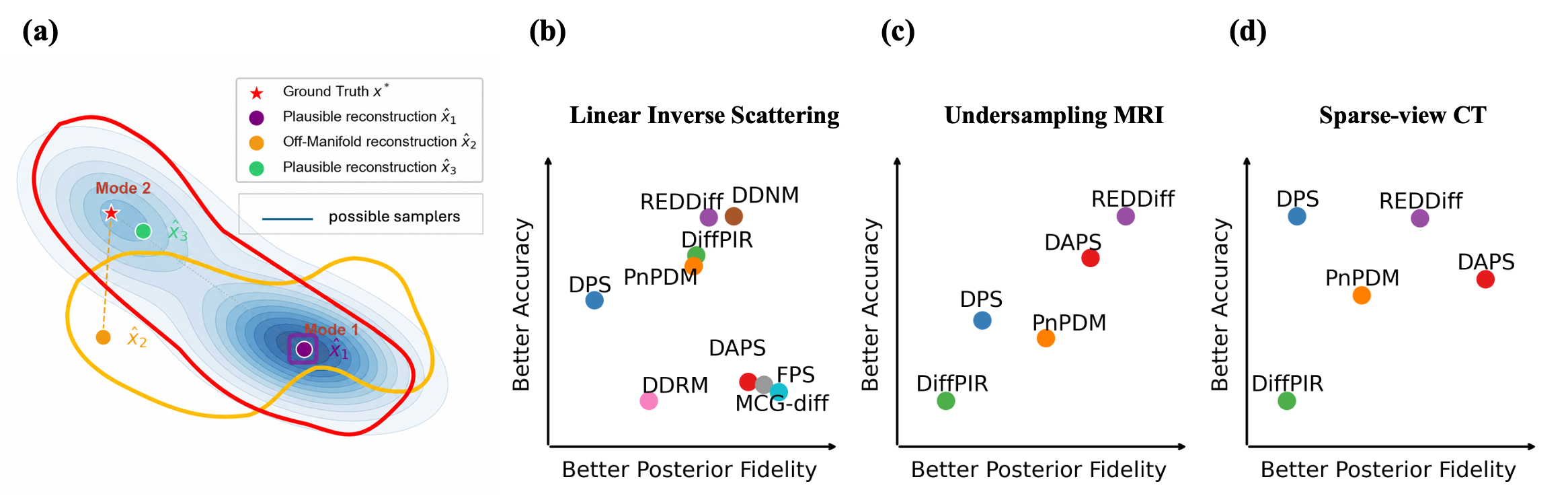}
    \caption{\textbf{(a): Illustration of the \textit{Accuracy Trap} phenomenon} and distinct uncertainty behaviors by different DIS samplers. \textbf{(b)$\sim$(d): Demonstration of posterior fidelity and accuracy performance across various DIS algorithms} in three inverse problems: \textbf{(b)} linear inverse scattering, \textbf{(c)} under-sampling MRI, and \textbf{(d)} sparse-view CT reconstruction.}
    \label{fig:abs}
    \vspace{-20pt}
\end{figure}

However, another gap remains in evaluation objective. 
Natural image restoration tasks often reward pixel-wise accuracy (\eg Peak Signal-to-Noise Ratio (PSNR)) from a random reconstruction~\cite{evalNEURIPS2021_6e289439}.
In contrast, inverse problems are inherently ill-posed with measurement noise which can leads to multiple physically plausible solutions (Fig.~\ref{fig:abs}), naturally leading to statistical uncertainty quantification\cite{kaipio2005statistical,stuart2010inverse}. Moreover, such uncertainty analysis is especially important and required in engineering and scientific applications, \ie calibrated uncertainty that preserves all physically valid solutions and enables principled risk quantification~\cite{UQMRIedupuganti2020uncertaintyquantificationdeepmri,kumar2024multiinverseDesign}.
This creates a crucial gap across evaluation of existing DIS works: 
not only are we ignoring the inherent stochastic nature of DIS, but we are also overlooking the critical role of uncertainty behavior of the sampled distribution requested in scientific applications. This mismatch is evident as shown in Fig.~\ref{fig:abs}, where several DIS methods produce similar accuracy performance in reconstructions for the same task, yet induce markedly different distributional behaviors, reflecting distinct posterior fidelity.

We call this phenomenon the \textit{Accuracy Trap}. As illustrated in Fig.~\ref{fig:abs}, relying solely on point accuracy metrics (\eg PSNR) can fundamentally mischaracterize posterior samplers. For instance, an off-posterior reconstruction $\hat{x}_2$ may achieve a higher PSNR than a posterior-plausible reconstruction $\hat{x}_3$ simply because $\hat{x}_2$ happens to be closer to the ground-truth $x^*$. Moreover, different DISs can exhibit qualitatively different uncertainty behaviors. 
Some solvers may produce well-dispersed samples that largely reflect the posterior uncertainty, some may generate a mixture of posterior-plausible and off-posterior samples, and others may collapse to nearly deterministic outputs. 
Consequently, robust uncertainty quantification (UQ) is not an optional add-on, but a prerequisite for deploying DIS methods in risk-sensitive scientific applications.

Since stochastic inverse solvers represent uncertainty through posterior samples, evaluating how well their generated samples capture the target posterior distribution becomes an important aspects of uncertainty quantification. 
Aligned with this goal, some recent DIS methods make efforts to introduce provable samplers~\cite{wu2024PnPDM,cardoso2023monteMCGdiff,dou2024fpsdiffusion,coeurdoux2023plugandplaysplitgibbssampler}, and validate the posterior estimation on controlled simulations where the analytical posterior is known using metrics such as sliced Wasserstein distance \cite{cardoso2023monteMCGdiff}.
However, evaluating posterior fidelity in realistic inverse problems remains largely unsolved. Existing distributional metrics, such as FID and LPIPS, require samples from both compared distributions and therefore inapplicable to real-world inverse problem without ground truth posterior samples.

Encouragingly, UQ has received growing attention in machine learning, through  aleatoric uncertainty(AU) and epistemic uncertainty(EU) decomposition~\cite{ML_uncertainty_H_llermeier_2021}, single-model uncertainty estimation~\cite{hofman2024quantifyingAandEUproperscore,chan2024estimatingEUAU_signlemodel,BayesianMRILuo_2023}, uncertainty-based distribution shift detection for DIS~\cite{kim2025distributionshiftuncertaintyestimationinverse}, and controlled statistical benchmarking studies~\cite{zach2025statisticalbenchmarkdiffusionposterior}. Yet, to our knowledge, no existing work address the following central question: \textit{Can stochastic DIS recover the posterior $p(x\mid y)$, and how should we evaluate such posterior fidelity without true posterior sampler and density, as in real-world inverse problems?}

\vspace{-8pt}
\paragraph{Contributions.}
To address this challenge, we provide a systematic study and propose a new metric to evaluate the posterior fidelity for DIS methods:

\begin{itemize}
    \item  We conduct a systematic study of posterior fidelity for a broad range of DIS in controlled simulation settings with known analytical true posterior. Beyond reconstruction accuracy, we analyze how well generated samples capture the target posterior distribution and characterize the distributional behavior of different DIS methods.
    \item We propose score-based Kernel Stein Discrepancy (\textbf{\ksd}), a \textit{theoretically grounded} and \textit{ground-truth-free} metric for evaluating posterior consistency of DIS methods in inverse problem solving. The proposed metric measures agreement between generated samples and posterior score field induced by the forward model and learned diffusion prior, enabling posterior-aware evaluation even when exact posterior samplers or densities are unavailable.
    \item Through experiments on both toy models and real-world inverse problems, we demonstrate that \ksd{} provides meaningful diagnostics of posterior fidelity beyond reconstruction accuracy, revealing that \textit{strong reconstruction performance does not necessarily imply better posterior consistency} (Fig.~\ref{fig:abs}), highlighting the importance of distribution-aware evaluation for stochastic inverse solvers.
\end{itemize}

%% file: S02Background.tex
\section{Preliminarily and Background}

\subsection{Diffusion Models}
\label{sec:diffusion_pnpdp}
Diffusion models (DM) have demonstrated  extraordinary ability to generate high quality images~\cite{song2021scorebasedgenerativemodelingstochastic,DDPMho2020denoisingdiffusionprobabilisticmodels,ddimsong2022denoisingdiffusionimplicitmodels}.
A diffusion model defines a forward noising process that transforms clean data $x_0\sim p_{\mathrm{data}}$
into noisy variables $x_t$ for $t\in[0,T]$, and learns a network that enables reversing this process.
In practice, the training of diffusion model can be viewed as either (i) estimating \emph{a score function} 
$s_\theta(x_t,t)\approx \nabla_{x_t}\log p_t(x_t)$ as formulated in the score-based DM~\cite{song2021scorebasedgenerativemodelingstochastic}, 
or (ii) learning \emph{a denoiser} that predicts a clean image
$\hat x_0=\mathrm{Denoise}_\theta(x_t,t)$ from the noisy image $x_t$ as formulated in Denoising Diffusion Probabilistic Model~\cite{DDPMho2020denoisingdiffusionprobabilisticmodels}, where $t$ denotes the diffusion sampling steps.
Throughout, we view the diffusion model as an \emph{implicit distributional prior} that can be queried via the score function or denoising operations, when the prior density $\log p_\theta(x_0)$ is not available in closed form.

\subsection{Diffusion Priors for Inverse Problem Solving}
The inverse problem aims at reconstructing an unknown signal \( x \in \mathbb{R}^n \) 
based on the measurements \( y \in \mathbb{R}^m \). Formally, \( y \) derives from a forward process determined by $y = \mathcal{A}x + \epsilon,$ where $ \mathcal{A} $ can be either a linear operator, such as the Radon transform in sparse-view CT reconstruction and Fourier transform in accelerated MRI, or a nonlinear operator, such as the JPEG restoration encoder.
$ \mathcal{A} $ can also be either given or unknown. 
In this work, we focus on the situation where $A$ is given. The term \( \epsilon \) denotes random measurement noise.

Diffusion inverse solver (DIS) methods combine a pretrained diffusion model prior $p_\theta(x)$
with a known forward model to perform inference for the posterior 
$p_\theta(x\mid y)\ \propto\ p(y\mid x)\,p_\theta(x),$
where the prior term $p_\theta(x)$ comes from the diffusion model prior and the likelihood term $p(y\mid x)$ is determined by forward operator $\mathcal{A}$ and the noise model.
The likelihood term enforces measurement consistency by favoring reconstructions that yield high $p(y\mid x)$. In practice, DIS algorithms impose measurement consistency in the diffusion sampling trajectory via different mechanisms, including gradients~\cite{chung2022diffusionDPS,zhang2024improvingdiffusioninverseproblemDAPS,wu2024PnPDM}, projection~\cite{chung2022scoremri,jalal2021robustcompressedsensingmri,kawar2022denoisingDDRM,wang2022zeroDDNM,kawar2022denoisingDDRM}, sampling~\cite{cardoso2023monteMCGdiff,dou2024fpsdiffusion}, or other optimizations~\cite{mardani2023reddiff1,song2023pseudoinverseReddiff2}.

Prior work has proposed different taxonomies for diffusion-based inverse solvers depending on the different criterion,
such as algorithmic structure, optimization technique, or the type of inverse problems~\cite{daras2024surveydiffusionmodelsinverse,zheng2025inversebench}.
For example, InverseBench groups existing DIS methds mainly based on algorithmic structure, including linear guidance, general guidance, variable-splitting, variational Bayes, and sequential Monte Carlo
\cite{zheng2025inversebench}.
InverseBench further provides a comprehensive benchmark that evaluates reconstruction performance across diverse tasks using standard
accuracy metrics including PSNR and Structural Similarity Index Measure (SSIM)~\cite{zheng2025inversebench}. 
While this accuracy assessment provides useful insights, 
it does not evaluate the posterior fidelity to understand the uncertainty and distributional behavior of different DIS methods.


\subsection{Posterior Uncertainty in Inverse Problems}
\label{sec:uncertainty_sources}
Solutions to the ill-posed inverse problems are inherently uncertain due to incomplete measurements, measurement noise, and imperfect prior information~\cite{kaipio2005statistical,stuart2010inverse}. In machine learning literature, these uncertainties are commonly categorized into epistemic uncertainty (EU) arising from limited information or model uncertainty, and aleatoric uncertainty (AU) arising from intrinsic stochasticity in the measurement generation process~\cite{UncertaintyDeepLearningNIPS2017_2650d608,nagel2016unified}.

In diffusion-based inverse problems solving, AU is primarily induced by measurement noise, while EU is associated with information loss from the ill-posed forward operator, potential model specification or prior mismatch. Thus, intrinsic posterior distribution induced by the inverse problem should exhibit substantial uncertainty, particularly  under ill-posed measurement settings. Since stochastic DIS aims to characterize posterior uncertainty through generated samples, posterior fidelity naturally becomes a key criterion for evaluating whether the sampled distributions reflect the underlying posterior behavior induced by the inverse problem.

\subsection{Limitation on Current Evaluation Metrics}
\label{sec:limits_accuracy}
Existing work mainly benchmarks the reconstruction quality by accuracy (\eg PSNR/SSIM). 
While accuracy metrics remain necessary, they are insufficient for evaluating DIS methods. 
There are two fundamental reasons:
(i) in ill-posed inverse problems the \emph{target} is a posterior distribution $p(x\mid y)$ with many plausible reconstructions of the same measurement,
and (ii) most DIS algorithms are inherently \emph{stochastic}, producing a distribution of reconstructions rather than a single deterministic output.
Together, the object of interest is a distribution over reconstructions, motivating uncertainty-aware evaluation.

Although posterior fidelity of DIS has recently received increasing attention~\cite{cardoso2023monteMCGdiff,zach2025statisticalbenchmarkdiffusionposterior}, existing metric such as Wasserstein distance is primarily limited to controlled simulation settings where ground-truth posterior is accessible. Common distributional metrics in real images such as FID~\cite{FIDheusel2017gans} and LPIPS~\cite{LPIPSzhang2018unreasonable} require samples from both compared distributions, making them inapplicable to real-world inverse problems where neither true posterior samplers nor normalized posterior densities are accessible. This limitation highlights an urgent need for distributional posterior fidelity evaluation methods that do not rely on access to ground-truth posterior distribution or samples.

%% file: S03-framework.tex
\section{Posterior Fidelity Evaluation via Score-KSD}

\subsection{Posterior Score Approximation}

To evaluate posterior fidelity, we seek a metric that measures how well the sample distribution induced by a DIS matches the Bayesian posterior. In synthetic settings, this can be achieved by comparing to ground-truth posterior samples. However, such samples are unavailable in realistic inverse problems, making direct distributional comparison infeasible.

A key observation is that, although the posterior density $p(x \mid y)$ is intractable, its score can be computed up to approximation.
Using Bayes' rule $  p(x\mid y) \propto  p(y\mid x)p(x),$ the posterior can be decomposed into  $\nabla_x \log p(x\mid y) = \nabla_x \log p(y\mid x) + \nabla_x \log p(x)$ after taking $\log$ and gradient.

Assuming Gaussian measurement noise \(\varepsilon \sim \mathcal N(0,\sigma_y^2 I)\), the likelihood score is analytically available
$\nabla_x \log p(y\mid x) = \frac{1}{\sigma_y^2}
J_{\mathcal A}(x)^\top \bigl(y-\mathcal A(x)\bigr),$ where \(J_{\mathcal A}(x)\) is the Jacobian of \(\mathcal A\) and it reduces to \(\sigma_y^{-2}\mathcal A^\top(y-\mathcal Ax)\) in the linear inverse problem. 
Moreover, although the prior score on clean image $ \nabla_x \log p(x)$ is unavailable, it can be approximated using the pretrained diffusion model through the pretrained score function $s_\theta(x_t,t)$ at small diffusion sampling timestep $t.$ Specifically, for a collection of small diffusion times $\{t_k\}_{k=1}^K$, we perturb $x$ as $x_{t_k}=\alpha_{t_k}x+\sigma_{t_k}z_k,\quad z_k\sim\mathcal N(0,I),$ and average them to approximate the diffusion score for clean images:
$\widehat s_{\mathrm{prior}}(x)
= \frac{1}{K}\sum_{k=1}^K \alpha_{t_k} s_\theta(x_{t_k},t_k).$ This yields an approximated posterior score
$\widehat s_{\mathrm{p}}(x;y)
= \nabla_x\log p(y\mid x) + \widehat s_{\mathrm{prior}}(x).$ The practical approximation details are provided in Appendix~\ref{supp:ksd}.

\subsection{Kernel Stein Discrepancy}
\begin{wrapfigure}{r}{0.6\textwidth}
\vspace{-40pt}
\begin{minipage}{0.6\textwidth}
\begin{algorithm}[H]
\caption{Score-Based KSD for DIS}
\label{alg:ksd_analytical}
\begin{algorithmic}[1]
\Require $\{x_i\}_{i=1}^N$, $y$, $\mathcal A$, $s_\theta$

\For{$i=1,\ldots,N$}
\State $s_{\mathrm{lik}}(x_i)=\frac{1}{\sigma_y^2} \mathcal A^\top (y - \mathcal A x_i)$
\State $z_k\sim\mathcal N(0,I)$
\State $\hat s_{\mathrm{prior}}(x_i)=\frac{1}{K}\sum_{k=1}^K \alpha_{t_k}\, s_\theta(\alpha_{t_k}x_i+\sigma_{t_k}z_k,t_k)$
\State $\hat s_p(x_i)=s_{\mathrm{lik}}(x_i)+\hat s_{\mathrm{prior}}(x_i)$
\EndFor
\State Compute $u_p(x_i,x_j)$ using Equation ~\ref{eq:kernel}.
\State \Return $ \text{score-KSD} =\frac{1}{N}\sqrt{\sum_{i,j=1}^N u_p(x_i,x_j)/d}$

\end{algorithmic}
\end{algorithm}
\end{minipage}
\vspace{-1em}
\end{wrapfigure}

Given this approximated posterior score induced by the pretrained diffusion score function $s_\theta$, together with generated $N$ posterior samples from a DIS method $\{x_i\}_{i=1}^N$, we can evaluate its posterior fidelity without access to posterior samples by using Kernel Stein Discrepancy (KSD)~\cite{KSDliu2016kernelizedsteindiscrepancygoodnessoffit}. 
KSD provides a score-based measure of whether generated samples are consistent with the Stein identity associated with the target posterior distribution.

Let \(q(x\mid y)\) denote the implicit sample distribution induced by a DIS method, and let \(\widehat s_{\mathrm{p}}(x;y)\) denote the approximated posterior score. For a test function \(f:\mathbb R^d\to\mathbb R^d\), the Langevin Stein operator is
$\mathcal T_p f(x)=\widehat s_{\mathrm{p}}(x;y)^\top f(x)+\nabla_x\cdot f(x).$
Under standard regularity conditions, if \(X\sim p(x\mid y)\), then
$\mathbb E[\mathcal T_p f(X)] = 0$ (see Proposition~\ref{prop:stein identity}).
KSD measures the maximum violation of this identity over a reproducing kernel Hilbert space (RKHS):
$\mathrm{KSD}(q,p)=\sup_{\|f\|_{\mathcal H^d}\le 1}\mathbb E_{X\sim q}\left[\mathcal T_p f(X)\right].$
For empirical samples
$\hat q_N=\frac{1}{N}\sum_{i=1}^N \delta_{x_i}, \quad x_i\sim q(x\mid y),$
the squared KSD admits the closed-form empirical estimator 
$\mathrm{KSD}^2(\hat q_N,p)=\frac{1}{N^2}\sum_{i,j=1}^N u_p(x_i,x_j),$
where \begin{equation}
\small \label{eq:kernel}
    u_p(x_i,x_j)= s_p(x_i)^\top k(x_i,x_j)s_p(x_j) +
    s_p(x_i)^\top \nabla_{x_j} k(x_i,x_j) +s_p(x_j)^\top \nabla_{x_i} k(x_i,x_j)+ \mathrm{tr}\!\left(
        \nabla_{x_i}\nabla_{x_j} k(x_i,x_j)
    \right)
\end{equation} is the standard Stein kernel term depending on the posterior score and kernel function denoted as $k(x_i,x_j)$ (see Proposition~\ref{prop:closed form}). 
We use the inverse multiquadric (IMQ)  kernel in all experiments, with details provided in Appendix~\ref{supp:ksd}. Since the magnitude of empirical KSD depends on the data dimension of $x$, we applied a  normalization in our proposed metric:
$\mathrm{\ksd{}} = \frac{1}{N}\sqrt{\sum_{i,j=1}^N u_p(x_i,x_j)/d},$ where $d$ is the dimension of $x.$
 Throughout the paper,  \ksd{} refers to this empirical normalized quantity unless otherwise specified.
KSD is used as a posterior-consistency diagnostic for generated samples. Under suitable kernel conditions\cite{KSDliu2016kernelizedsteindiscrepancygoodnessoffit,gorham2019measuringsamplequalitysteins}, KSD is nonnegative and equals zero if and only if the sample distribution matches the target posterior distribution (see Proposition~\ref{prop:valid measure}). Consequently, within a fixed inverse problem setup, a smaller \ksd{} generally indicates stronger consistency between the generated sample distribution and the target posterior score field.
Note that the absolute magnitude of \ksd{} depends on posterior sharpness, dimensionality, etc.\ Therefore, \ksd{} should be interpreted as a within-task posterior-consistency diagnostic to evaluate posterior fidelity, rather than an absolute cross-task metric.

\begin{proposition}[Stein identity for the posterior~\cite{gorham2019measuringsamplequalitysteins,KSDliu2016kernelizedsteindiscrepancygoodnessoffit}]
\label{prop:stein identity}
Let \(p(x\mid y)\) be a differentiable posterior density on \(\mathbb R^d\), and define its score as $s_p(x) := \nabla_x \log p(x\mid y).$
For a vector-valued test function \(f:\mathbb R^d\to\mathbb R^d\), define the Langevin Stein operator
  $\mathcal T_p f(x) =s_p(x)^\top f(x) +\nabla_x\cdot f(x).$
Assume \(f\) is sufficiently smooth and satisfies the boundary condition
$ \lim_{\|x\|\to\infty} p(x\mid y) f(x)=0,$
so that integration by parts is valid. Then, if \(X\sim p(x\mid y)\),
\[
    \mathbb E_{X\sim p(x\mid y)}
    \left[
        \mathcal T_p f(X)
    \right]=0.
\]
\end{proposition}

\begin{proposition}[KSD is a valid discrepancy measure]
\label{prop:valid measure}
Kernel Stein Discrepancy satisfies the following properties:
\begin{enumerate}
    \item \textbf{Non-negativity:}
        $\mathrm{KSD}(q,p) \ge 0.$
    \item \textbf{Identity of indiscernibles:} Under suitable smoothness and integrability conditions on \(p\)~\cite{KSDliu2016kernelizedsteindiscrepancygoodnessoffit,gong2021slicedkernelizedsteindiscrepancy}, and for a characteristic kernel \(k\),
        $\mathrm{KSD}(q,p)=0
        \quad\Longleftrightarrow\quad
        q(x\mid y)=p(x\mid y).$
\end{enumerate}
\end{proposition}

\begin{proposition}[Closed-form KSD with empirical distribution]
\label{prop:closed form}
Let \(p(x\mid y)\) be the target posterior with score
   $ s_p(x)=\nabla_x\log p(x\mid y),$
and let \(q(x\mid y)\) be the sample posterior distribution induced by a sampler. Given samples
    $x_i\sim q(x\mid y),\quad i=1,\dots,N,$
define the empirical distribution
    $\hat q_N=\frac{1}{N}\sum_{i=1}^N \delta_{x_i}.$
Let \(\mathcal H\) be an RKHS with scalar kernel \(k\), and let
\(\mathcal H^d\) be the corresponding vector-valued RKHS. The KSD between
\(\hat q_N\) and \(p\) is
    $\mathrm{KSD}(\hat q_N,p)
    =
    \sup_{\|f\|_{\mathcal H^d}\le 1}
    \frac{1}{N}\sum_{i=1}^N
    \mathcal T_p f(x_i),$
where
    $\mathcal T_p f(x)
    =
    s_p(x)^\top f(x)
    +
    \nabla_x\cdot f(x).$
Then the squared empirical KSD admits the closed-form expression $$\mathrm{KSD}^2(\hat q_N,p)= \frac{1}{N^2} \sum_{i=1}^N \sum_{j=1}^N u_p(x_i,x_j),$$ where $u_p(x_i,x_j)
    =
    s_p(x_i)^\top k(x_i,x_j)s_p(x_j) +
    s_p(x_i)^\top \nabla_{x_j} k(x_i,x_j) +
    s_p(x_j)^\top \nabla_{x_i} k(x_i,x_j)
    +
    \mathrm{tr}\!\left(
        \nabla_{x_i}\nabla_{x_j} k(x_i,x_j)
    \right)$

\end{proposition}

%% file: S04ToySimulation.tex
\section{Numerical Simulation Study}
\subsection{Qualitative Analysis of Posterior Behavior}
The first emphasis of this work is to understand the distributional behavior of different DIS methods in inverse problem settings. To this end, we conduct a numerical study using a mixture-of-Gaussians prior under the noisy linear inverse problem $y = \mathcal A x + \epsilon$, for which the analytical posterior density is available. We visualize posterior sample behavior through pairwise scatter plots and compare the generated samples against ground-truth posterior samples. These visualizations provide an intuitive assessment of posterior fidelity, including recovery of the overall posterior geometry, correlation structure, concentration, and mode coverage. Detailed experiment settings are provided in Supp~\ref{supp:nume}.
\begin{figure}[t]
    \includegraphics[width=\linewidth]{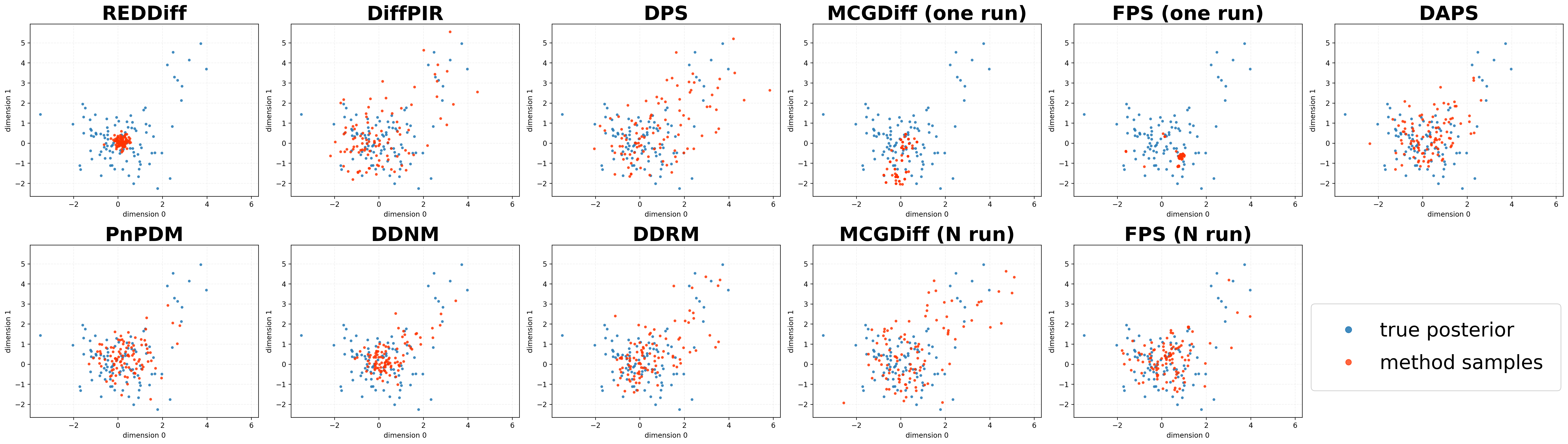}
    \caption{\textbf{Posterior sample comparison in toy model experiments} with an inverse problem setting $(d_x,d_y)=(16,14)$.
The scatter plots visualize the unobserved dimensions 0 and 1, whose prior follows a two-mode mixture-of-Gaussians distribution with modes centered at $(0,0)$ and $(3,3)$ with sample size $N=500,$ noise scale $\sigma=0.2$.
Blue and red points denote samples from the ground-truth posterior sampler and each diffusion sampler, respectively.\protect\footnotemark}
    \label{fig:num-scatter}
    \vspace{-10pt}
\end{figure}

\footnotetext{
\footnotesize
\textbf{Note:} \textsc{MCGDiff} and \textsc{FPS} use all returned particles from a single run, while \textsc{MCG\_Mul} and \textsc{FPS\_Mul} run the sampler independently $N$ times and retains the best particle from each run, following the evaluation scheme adopted in the original~\textsc{MCGDiff}~\cite{cardoso2023monteMCGdiff} experiments setting.
}
As shown in Fig.~\ref{fig:num-scatter}, some DIS methods fail to recover the weaker mode, while others preserve multimodal structure. Moreover, even within the same mode, different DIS methods exhibit diverse sample concentration behavior: some collapse a small limited region while others produce more dispersed samples that better align with the true posterior. These observations demonstrate that DIS methods have fundamentally different posterior behaviors despite generated reconstruction samples mostly fall in the plausible posterior region, highlighting the necessity of posterior fidelity evaluation beyond accuracy alone.

\subsection{Empirical KSD Finite Particle Analysis}
To validate our proposed score-KSD as a posterior fidelity diagnostic, we first study its finite-sample behavior using posterior samples in this numerical study. Although the population-level KSD of the true posterior satisfies $\mathrm{KSD}(q,p) = 0$ as described in Proposition~\ref{prop:valid measure}, the empirical score-KSD computed from a finite number of posterior samples is generally nonzero due to finite sample effects. We therefore investigate how score-KSD behaves with respect to sample size, observation strength, and measurement noise.

\begin{figure}[h]
\vspace{-1em}
    \centering
    \includegraphics[width=0.85\linewidth]{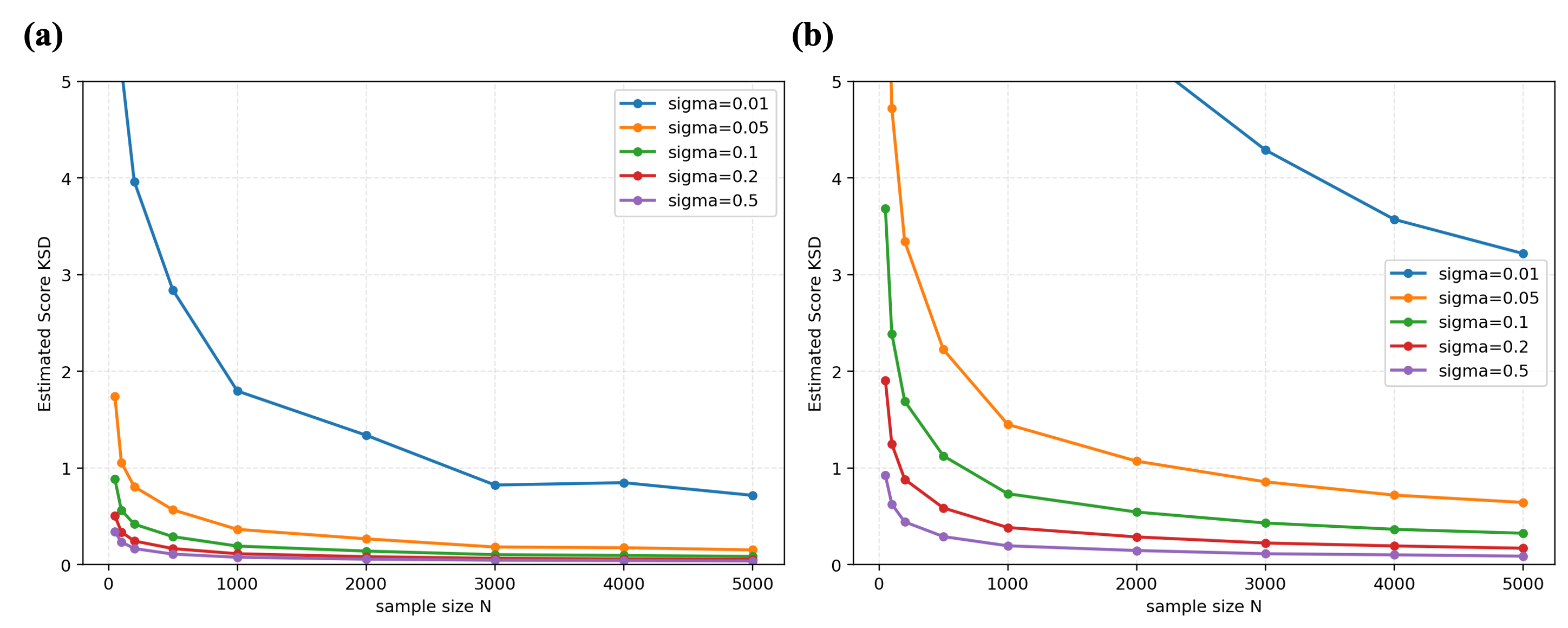}
    \caption{(a): Score-KSD curve of finite posterior samples under  $(d_y=14)$ and observations scale $\in[0.1, 0.75]$ with varying measurement noise scales, (b): Score-KSD curve of finite posterior samples under $(d_y=4)$ and observations scale $ = 3$ with varying measurement noise scales.}
    \label{fig:combine_high_low}
\end{figure}
In Fig.~\ref{fig:combine_high_low}, empirical score-KSD decreases monotonically as the number of samples N increases, approaching zero as the empirical distribution better approximates the true posterior. 
Moreover, larger measurement noise $\sigma_y$ and observation settings with weaker constraints or sparser observations both lead to smaller empirical score-KSD values, since they induce smoother and less sharply concentrated posterior geometries with reduced posterior-score magnitude, thereby reducing the finite-sample variability of score-KSD. 
Importantly, due to finite-sample effects, the empirical score-KSD computed from true posterior samples should not be interpreted as a strict lower bound of the metric, but rather as a finite-sample reference baseline in a controlled inverse-problem setting.
\subsection{Score-KSD Aligns with Posterior Visualization}
\vspace{-6pt}
After characterizing the finite-sample behavior of score-KSD, we next investigate whether score-KSD can meaningfully detect posterior fidelity of different DIS methods. We compute score-KSD using both the analytical posterior score derived from the exact posterior density and the approximate posterior score constructed from the likelihood model and the learned diffusion prior for each experimental setting.
We further compare the numerical score-KSD values of $\sigma=0.2$ in Table~\ref{tab:num_results_n500} with its posterior scatter visualizations in Fig.~\ref{fig:num-scatter}. We observe that methods exhibiting severe posterior mismatch, such as mode collapse or failure to recover weaker posterior modes, consistently produce larger score-KSD values using analytical score (e.g., RED-Diff: 1.57,  FPS (one run): 1.77). In contrast, methods that successfully recover both posterior modes and generate samples whose geometry better aligns with the analytical posterior achieve considerably smaller values (e.g., MCG-Diff (N run): 0.28, DiffPIR: 0.50, DDRM: 0.51). This qualitative consistency between the scatter visualizations and the corresponding score-KSD rankings provides empirical evidence that score-KSD meaningfully captures posterior-consistency behavior across different DIS methods.
\begin{table*}[t]
\caption{Root Mean Square Error (\textbf{RMSE}) for accuracy evaluation, \ksd{} using analytical posterior score (\textbf{An-KSD}), and score-KSD using approximate posterior score (\textbf{Ap-KSD}) under different noise levels using sample size $N=500$ with many weak measurements. Results are reported as mean and standard deviation across 5 noise draws generated from $N\sim (0,\sigma^2)$.}
\centering
\small
\resizebox{0.8\textwidth}{!}{
\setlength{\tabcolsep}{4pt}
\begin{tabular}{l|ccc|ccc}
\toprule
& \multicolumn{3}{c|}{$\boldsymbol{\sigma=0.2}$} 
& \multicolumn{3}{c}{$\boldsymbol{\sigma=0.5}$} \\
\cmidrule(lr){2-4} \cmidrule(lr){5-7}
\textbf{Method} 
& \textbf{RMSE $\downarrow$} & \textbf{An-KSD $\downarrow$} & \textbf{Ap-KSD $\downarrow$}
& \textbf{RMSE $\downarrow$} & \textbf{An-KSD $\downarrow$} & \textbf{Ap-KSD $\downarrow$} \\
\midrule
DAPS\cite{zhang2024improvingdiffusioninverseproblemDAPS}
& \textbf{1.04} (0.06) & 0.64 (0.05) & 0.63 (0.04)
& 1.57 (0.29) & 1.10 (0.21) & 1.04 (0.19) \\

DDNM\cite{wang2022zeroDDNM}
& 1.23 (0.25) & 1.23 (0.21) & 1.20 (0.18)
& 1.86 (0.42) & 2.12 (0.46) & 2.02 (0.41) \\

DDRM\cite{kawar2022denoisingDDRM}
& 1.08 (0.10) & 0.51 (0.05) & 0.50 (0.05)
& 1.29 (0.17) & 0.60 (0.10) & 0.57 (0.09) \\

DiffPIR\cite{zhu2023denoisingDiffPIR}
& 1.20 (0.11) & 0.50 (0.07) & 0.49 (0.06)
& 1.69 (0.28) & 0.99 (0.22) & 0.93 (0.19) \\

DPS\cite{chung2022diffusionDPS}
& 1.14 (0.01) & 0.74 (0.12) & 0.74 (0.13)
& \textbf{1.23} (0.04) & \textbf{0.25} (0.04) & \textbf{0.25} (0.04) \\
FPS (N runs)\cite{dou2024fpsdiffusion}
& 1.27 (0.23) & 0.96 (0.17) & 0.93 (0.15)
& 1.76 (0.34) & 1.20 (0.29) & 1.13 (0.26) \\
FPS (one run)
& 1.21 (0.21) & 1.77 (0.29) & 1.74 (0.27)
& 1.83 (0.32) & 1.71 (0.47) & 1.62 (0.42) \\

MCG-Diff (N runs)\cite{cardoso2023monteMCGdiff}
& 1.07 (0.03) & \textbf{0.28} (0.01) & \textbf{0.28} (0.01)
& 1.26 (0.08) & \textbf{0.25} (0.02) & \textbf{0.25} (0.02) \\
MCG-Diff (one run)
& 1.21 (0.21) & 1.09 (0.17) & 1.09 (0.17)
& 1.37 (0.15) & 0.85 (0.30) & 0.84 (0.29) \\

PnPDM\cite{wu2024PnPDM}
& 1.19 (0.18) & 1.04 (0.13) & 1.02 (0.10)
& 1.83 (0.38) & 1.61 (0.35) & 1.52 (0.31) \\
RED-Diff\cite{mardani2023reddiff1,song2023pseudoinverseReddiff2}
& 1.11 (0.08) & 1.57 (0.05) & 1.56 (0.05)
& 1.65 (0.29) & 2.14 (0.29) & 2.06 (0.26) \\
\hline
\rowcolor[HTML]{C0C0C0}
Finite Posterior Reference
& 1.13 (0.04) & 0.35 (0.00) &0.35 (0.00)
& 1.30 (0.06) & 0.24 (0.00) & 0.24 (0.00) \\
\bottomrule
\end{tabular}
}
\label{tab:num_results_n500}
\vspace{-16pt}
\end{table*}

Moreover, score-KSD computed using the approximate posterior score is close to the score-KSD using the analytical posterior score across different methods and noise scales, supporting that our proposed posterior score approximation based on the likelihood model and learned diffusion prior provides a practical and effective tool for posterior-consistency evaluation when the analytical posterior score is unavailable. Finally while posterior reference samples provide an important calibration baseline, they do not necessarily attain the minimum \ksd{} due to finite-sample error. In particular, some samplers may produce more regular or score-consistent finite sample sets under the chosen kernel, leading to slightly smaller \ksd{}  values than finite posterior samples. We therefore interpret \ksd{}  primarily as a posterior-consistency diagnostic based on finite-sample score information within the same task under finite samples, rather than as an absolute population-level discrepancy metric.

%% file: S05RealData.tex
\section{Real Data Experiments}

\begin{figure}[t]
\centering
\begin{subtable}[c]{0.5\linewidth}
\label{tab:InvScatteringComparison}
\resizebox{\linewidth}{!}{
\centering
\begin{tabular}{lcc|cc}
\toprule
& \multicolumn{4}{c}{\textbf{Linear Inverse Scattering ($\boldsymbol{\sigma=0.0001}$)}} \\
\cmidrule(lr){2-5}
& \multicolumn{2}{c|}{\textbf{180 views}} 
& \multicolumn{2}{c}{\textbf{360 views}}\\
\cmidrule(lr){2-3}\cmidrule(lr){4-5}
\textbf{Method} 
& \textbf{PSNR(std) $\uparrow$} & \textbf{KSD $\downarrow$}
& \textbf{PSNR(std) $\uparrow$} & \textbf{KSD $\downarrow$}\\
\midrule
DAPS          & 27.81(0.10) & 3.70    & 29.21(0.12) & 5.63    \\
DDNM          & \textbf{35.14}(0.10) & 5.04    & \textbf{36.25}(0.11) & 8.04    \\
DDRM          & 26.97(0.01) & 30.53   & 31.17(0.05) & 21.92   \\
DPS           & 31.42(0.19) & 96.95   & 31.65(0.19) & 234.83  \\
DiffPIR       & 33.41(0.14) & 11.18   & 33.64(0.15) & 19.56   \\
FPS           & 27.69(0.02) & \underline{2.65}    & 30.45(0.08) & \underline{3.58}    \\
MCG-Diff      & 27.36(0.03) & \textbf{1.94}    & 29.52(0.13) & \textbf{2.18}    \\
PnPDM         & 32.94(0.16) & 11.81   & 34.83(0.16) & 18.00   \\
RED-Diff       & \underline{35.09}(0.09) & 8.59    & \underline{36.24}(0.10) & 11.19   \\
\hline
\rowcolor[HTML]{C0C0C0}
Uncond.       & 8.98(0.77)  & 901.65  & 8.98(0.77)  & 1783.04 \\
\rowcolor[HTML]{C0C0C0}
Noise         & 12.26(0.04) & 3182.67 & 12.26(0.04) & 6131.81 \\
\bottomrule
\end{tabular}
}
\caption{Results for linear inverse scattering task.}
\end{subtable}
\hfill
\begin{subfigure}[c]{0.48\linewidth}
\centering
\vspace{2.5em}
\includegraphics[width=\linewidth]{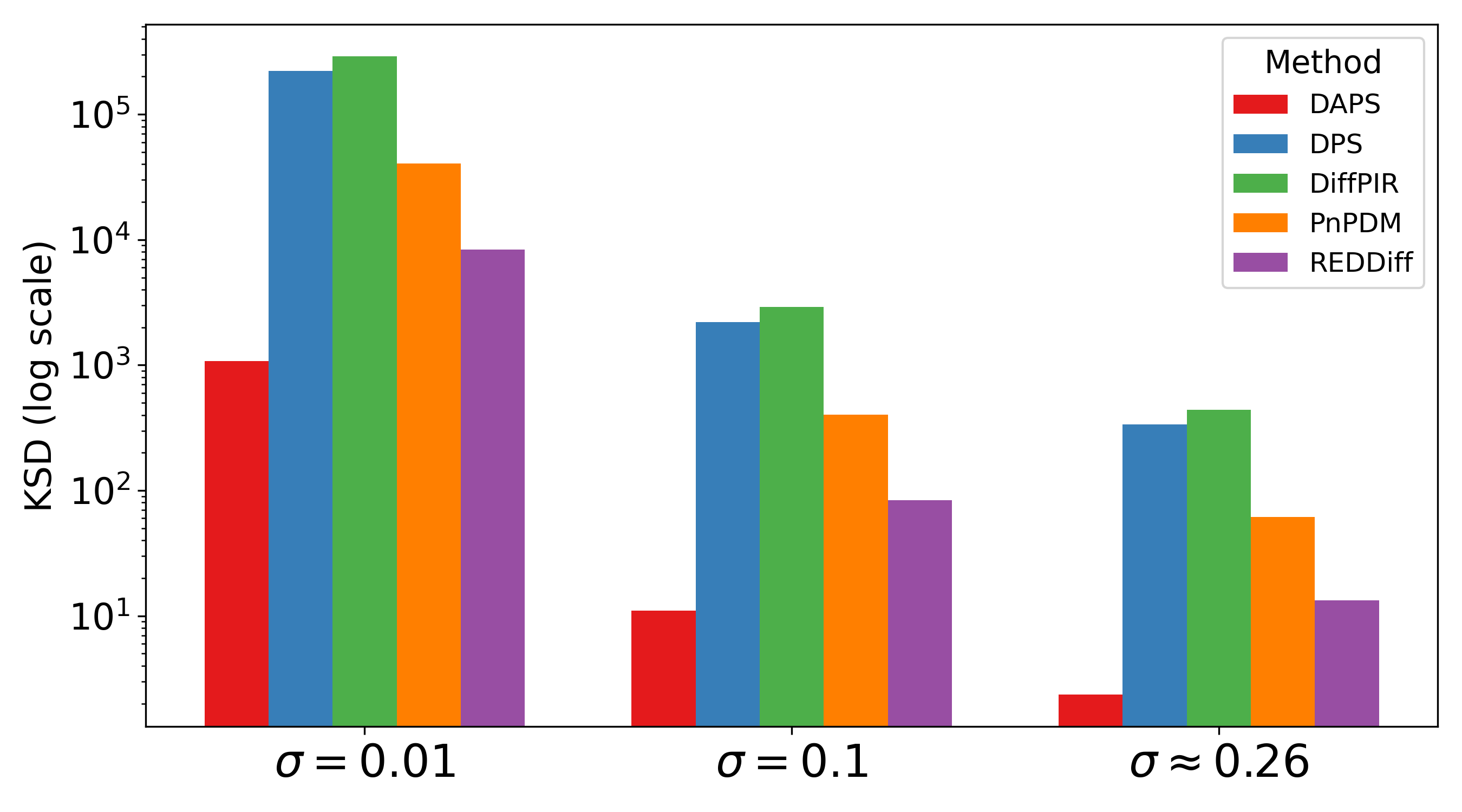}
\caption{Score-KSD with various measurement noise scales in 20-view CT reconstruction task.}
\label{fig:ct_20_views}
\end{subfigure}
\caption{Performance comparison and score-KSD behavior.}
\label{fig:table_figure_side_by_side}
\vspace{-6pt}
\end{figure}

\begin{table*}[t]
\caption{Results comparison of different DIS methods in MRI and CT reconstruction tasks (averaged value across 50 samples for one target image). Experiments are held on MRI measurements degraded with $\sigma=0.01$, in-distribution (ID) CT measurements degraded with $\sigma=0.1$, and out-of-distribution(OOD) CT measurements degraded with $\sigma=0.1$. (\textbf{Bold} marks the best value for each reported metric, and \underline{underline} marks the second-best value.)}
\label{tab:CTMRIInvcomparison}
\resizebox{\textwidth}{!}{
\begin{tabular}{lcc|cc|cc|cc|cc}
\toprule
& \multicolumn{4}{c|}{\textbf{MRI}} 
& \multicolumn{4}{c|}{\textbf{CT (ID)}} 
& \multicolumn{2}{c}{\textbf{CT (OOD)}} \\
\cmidrule(lr){2-5}\cmidrule(lr){6-9}\cmidrule(lr){10-11}
& \multicolumn{2}{c|}{\textbf{AR = 8}} 
& \multicolumn{2}{c|}{\textbf{AR = 4}}
& \multicolumn{2}{c|}{\textbf{20 views}} 
& \multicolumn{2}{c|}{\textbf{60 views}}
& \multicolumn{2}{c}{\textbf{20 views}}\\
\cmidrule(lr){2-3}\cmidrule(lr){4-5}\cmidrule(lr){6-7}\cmidrule(lr){8-9}\cmidrule(lr){10-11}
Method 
& \textbf{PSNR(std) $\uparrow$} & \textbf{KSD $\downarrow$}
& \textbf{PSNR(std) $\uparrow$} & \textbf{KSD $\downarrow$}
& \textbf{PSNR(std) $\uparrow$} & \textbf{KSD $\downarrow$}
& \textbf{PSNR(std) $\uparrow$} & \textbf{KSD $\downarrow$}
& \textbf{PSNR(std) $\uparrow$} & \textbf{KSD $\downarrow$}\\
\midrule
DAPS    & \underline{30.38}(0.18) & \underline{4.92} & \underline{33.00}(0.05) & \underline{7.24}   &  28.23(0.06) & \textbf{11.01} & \underline{35.15}(0.08) & \textbf{48.96}    & \underline{25.03}(0.06) & \textbf{17.44} \\
DiffPIR & 24.32(0.35) & 95.00  &25.52(0.27) & 127.58 &  22.02(0.17) & 2903.88 & 23.18(0.12) & 16291.87 & 20.20(0.09) & 3407.80 \\
DPS     & 27.77(0.44) & 41.40  &29.73(0.13) & 62.14  &  \textbf{31.52}(0.19) & 2211.65 & 34.4(0.33)  & 14777.64 & 24.79(0.21) & 2108.72 \\
PnPDM   & 28.15(0.03) & 11.62  &28.81(0.02) & 17.59  &  27.40(0.48) & 404.25  & 32.11(0.07) & 756.99   & 24.36(0.14) & 382.65 \\
RED-Diff & \textbf{32.66}(0.06) & \textbf{2.82}   &\textbf{35.18}(0.03) & \textbf{3.61}   &  \underline{31.31}(0.09) & \underline{83.97}  & \textbf{37.75}(0.05) & \underline{139.54}   & \textbf{25.36}(0.07) & \underline{111.31} \\
\hline
\rowcolor[HTML]{C0C0C0}
Uncond. & 6.68(0.01)  & 1315.57 &6.68(0.01)  & 946.50 &  14.16(1.24) & 20918.91 &    14.16(1.24)&77759.81&14.16(1.24) & 20918.91  \\
\rowcolor[HTML]{C0C0C0}
Noise   & 18.42(1.07) & 1363.02 &18.17(1.02) & 836.29 &  5.43(0.02)  & 102111.27 &  5.43(0.02) & 379789.52 &5.43(0.02)  & 102111.27 \\
\bottomrule
\end{tabular}
}
\vspace{-10pt}
\end{table*}
\subsection{Experiment Setup}
\label{sec:real_setup}
\vspace{-6pt}
\paragraph{Tasks and Datasets.}
We evaluate the posterior fidelity performance of DIS methods through our proposed score-KSD on three representative real-data inverse problems:
(i) \emph{linear inverse scattering},
(ii) \emph{under-sampling MRI reconstruction},
and (iii) \emph{sparse-view CT reconstruction}.

For the linear inverse scattering (data from \cite{wiesner2019cytopacqlinearinverse}) and multi-coil MRI (fastMRI knee data from \citep{zbontar2019fastmriopendatasetbenchmarks}), we follow the corresponding experimental setups in InverseBench~\cite{zheng2025inversebench}. 
For inverse scattering, we consider the number of receivers $M = 180, 360$ and the noise scale $\sigma = 0.0001$,
while for sparse-sampling MRI, we evaluate $\times4$ and $\times8$ acceleration rate (AR) and noise scale $\sigma=0.01$. 

For sparse-view CT (SVCT) task, we conduct experiments using the LIDC-IDRI dataset~\cite{lidc}. The original CT volumes are resampled to a slice thickness of 1 mm, and each slice is resized to $256\times256$. The training set consists of 23,040 images, and in-distribution evaluation is conducted on the hold-out data.
The diffusion model is trained using the pipeline proposed in~\cite{EDM} and the same trained model is used for all PnPDP methods. For out-of-distribution (OOD) evaluation, we use Lung-PET-CT-Dx dataset~\cite{Lung-PET-CT-Dx} from cancer patients. 
We directly use the pretrained diffusion models from LIDC-IDRI dataset as the prior for reconstructing images from Lung-PET-CT-Dx dataset without any adaptation, thus as OOD task with imperfect or mismatch priors.

\vspace{-10pt}
\paragraph{Evaluation Procedure.}
For each task, we first sample a noise $\epsilon\sim N(0,\sigma^2),$ and generate the simulated observation $y = Ax +\epsilon.$ We run each DIS method $N=50$ times with different random seeds to generate posterior samples for the simulated observation $y$. We evaluation reconstruction accuracy using PSNR and accessing posterior fidelity using the proposed score-KSD metric. See more details for hyperparameter settings in Appendix~\ref{supp:hyper_real}.

\subsection{Results and Findings}
\paragraph{Score-KSD distinguishes meaningful posterior behavior from trivial baselines.}
We add the unconditional prior sampling and pure noise images as the trivial baseline for an intuitive comparison. Across all real-world inverse problem tasks in Table~\ref{tab:CTMRIInvcomparison}, all DIS methods achieve substantially smaller score-KSD values than these two trivial baselines within the same task. This indicates that \ksd{} meaningfully captures the posterior behaviors.
\vspace{-6pt}
\paragraph{Score-KSD ranking exhibits partial cross-task consistency.}
We observe partially consistent score-KSD rankings across different inverse-problem tasks. Some DIS methods consistently achieve better performance in \ksd{} across multiple settings, as demonstrated in Fig.~\ref{fig:abs}(b)-(d)), suggesting a stable posterior fidelity behavior. Meanwhile, we also observe that the \ksd{} rankings remain task-dependent, consistent with the fact that \ksd{} is a within-task posterior-consistency diagnostic, since posterior score can vary substantially with forward operators, noise scales, \emph{etc}.
\vspace{-6pt}
\paragraph{Score-KSD is stable for different test images within the same task.}
Although \ksd{} values are not directly comparable across different inverse problems, we observe stable \ksd{} behavior across different test images within the same task setting as shown in Appendix (Table~\ref{tab:combined_ksd_ct} and~\ref{tab:combined_ksd_ct_cancer}). This finding supports the robustness of \ksd{} as a within-task posterior-consistency diagnostic.
\vspace{-6pt}
\paragraph{Score-KSD captures distributional behavior beyond accuracy.} 
Methods with similar accuracy can exhibit substantially different \ksd{} values (Fig.~\ref{fig:abs}(b)-(d)) and pixel variance maps (Fig.~\ref{fig:mri-ar4}), and we do not observe any monotonic relationship in which better reconstruction accuracy necessarily corresponds to better posterior fidelity. These results highlight that accuracy alone fails to fully characterize the behavior of stochastic DIS algorithms, and our proposed \ksd{} serves as an important complementary metric for evaluating posterior consistency behavior beyond accuracy.
\vspace{-6pt}
\paragraph{Ablation Study on OOD task and hyperparamter sensitivity.}
We further explored OOD inverse problems and hyperparameter sensitivity, discovering that DPS is highly sensitive to hyperparameter choices while DAPS requires hyperparameter adjustment to obtain reasonable reconstruction quality (Table~\ref{tab:dps-hyper} and Sec.~\ref{supp:hyper_real} in Appendix). OOD settings consistently lead to degraded reconstruction quality together with larger \ksd{} values, indicating worse posterior consistency (Table~\ref{tab:CTMRIInvcomparison}).
\begin{figure}
    \centering
    \includegraphics[width=\linewidth]{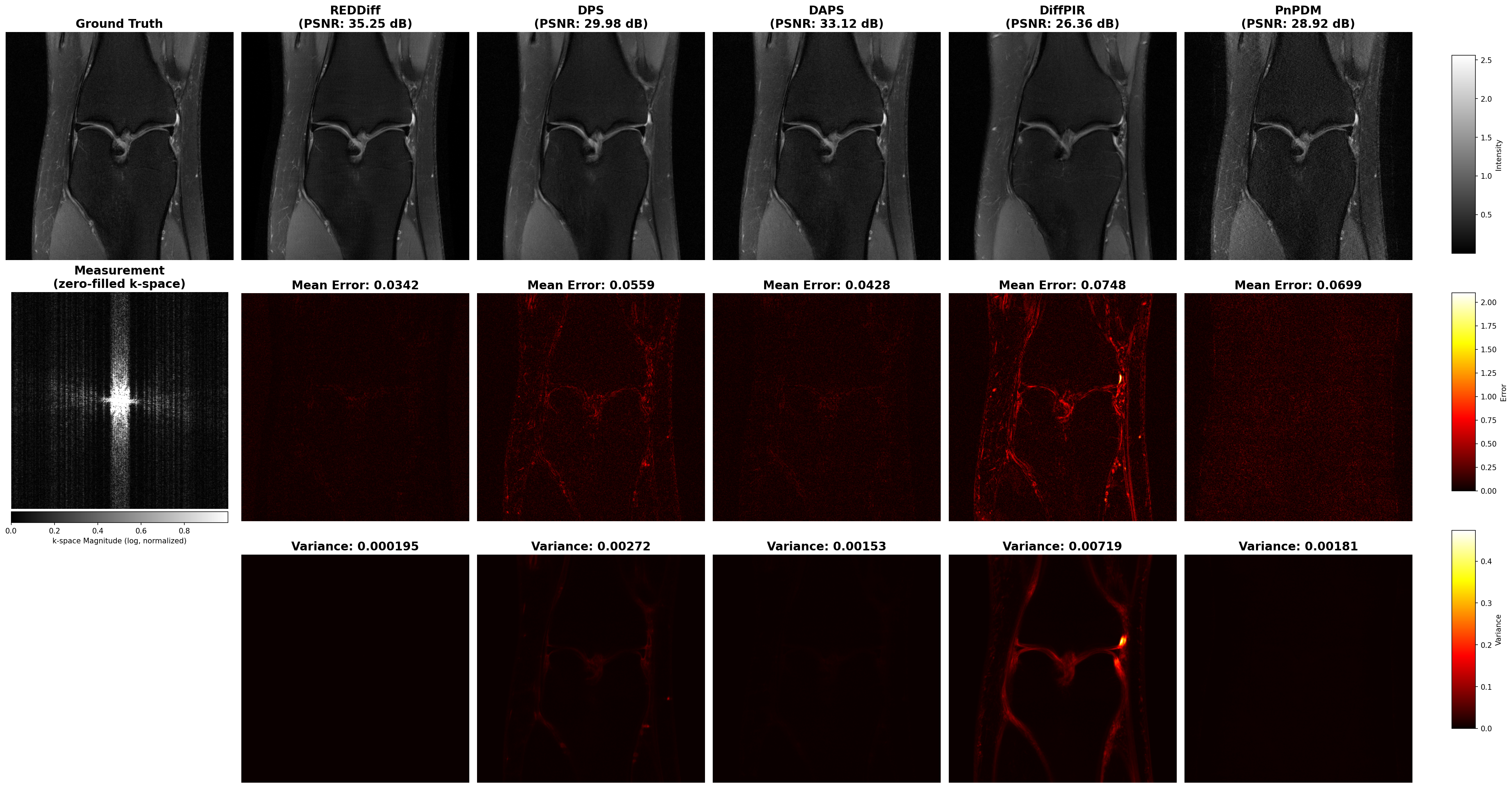}
    \caption{Under-sampling MRI reconstruction under $\times4$ acceleration rate (AR=4).}
    \label{fig:mri-ar4}
\end{figure}

%% file: S06Discussion.tex
\section{Discussion and Conclusion}

In this work, we identify the critical \textit{"Accuracy Trap}" in DIS benchmarking, and study their posterior behavior from a distributional perspective. Motivated by that, we propose the theory-grounded and ground-truth-free \ksd{} as a posterior-aware diagnostic for evaluating posterior fidelity. Through controlled simulations, and real-world inverse problems, our results suggest that \ksd{} constructed from the likelihood model and learned diffusion prior, provides a practical and meaningful tool for posterior-consistency evaluation when neither true posterior samples nor posterior density are accessible. 
One limitation of the proposed \ksd{} framework is that it requires  the noise scale $\sigma_y$ in the inverse problem, which may not be directly accessible in practice and require additional estimation. How inaccuracies in such estimates affect the \ksd{} evaluation remains an important direction for future work.

%% file: supp-ksd.tex
\newpage
\section{Details for \ksd{}}
\label{supp:ksd}
 
\subsection{IMQ Kernel }
We use the inverse multiquadric (IMQ) kernel
\[
k(z,z')
=
\left(c^2+\|z-z'\|_2^2\right)^\beta,
\qquad \beta\in(-1,0),
\]
with \(\beta=-1/2\). The scale parameter \(c\) is chosen adaptively as
$
c=\frac{1}{\operatorname{median}(s(A))+1},
$ to keep a fair comparison across each method within the same task, where \(s(A)\) denotes the singular values of the forward operator \(A\).

The IMQ kernel is widely used in the KSD literature due to its strong empirical performance and favorable theoretical properties for detecting distributional mismatch, particularly in the tails \cite{gorham2017measuring}. While the choice of kernel is not unique and can influence the absolute KSD values and relative rankings, the IMQ kernel provides a robust and sensitive measure of posterior inconsistency in practice.

\subsection{Prior Score Approximation Details}
For approximate posterior score in \ksd{} computation, we approximate the clean prior score using the pretrained diffusion score network evaluated near the clean-data limit. In our EDM implementation, the network is parameterized by the noise scale \(\sigma\) with noise level \(\sigma_{\mathrm{score}}=0.3\) and draw \(M=4\) independent Gaussian perturbations for each sample \(x\). The approximate prior score is computed as
\[
\widehat s_{\theta,0}(x)
=
\frac{1}{M}\sum_{m=1}^M
\alpha_{\sigma_{\mathrm{score}}}
s_\theta
\left(
\alpha_{\sigma_{\mathrm{score}}}x
+
\sigma_{\mathrm{score}} z_m,
\sigma_{\mathrm{score}}
\right),
\qquad
z_m\sim\mathcal N(0,I).
\]
We then construct the approximate posterior score by combining the likelihood score with the approximated prior score.

%% file: supp-simulation.tex
\section{Numerical Simulations Settings and Additional Results}
\label{supp:nume}
\paragraph{Prior Distribution}

We define a structured prior distribution over the unknown signal

$x \in \mathbb{R}^{16}.$

The first two coordinates \((x_1, x_2)\) follow a two-component Gaussian mixture model, while the remaining coordinates are modeled as independent standard Gaussian variables.

Specifically, let
\[
x =
\begin{bmatrix}
x_{\mathrm{mog}} \\
x_{\mathrm{tail}}
\end{bmatrix},
\quad
x_{\mathrm{mog}} = (x_1, x_2)^\top \in \mathbb{R}^2,
\quad
x_{\mathrm{tail}} = (x_3,\dots,x_{16})^\top \in \mathbb{R}^{14}.
\]

The mixture prior on the first two dimensions is
\[
p(x_{\mathrm{mog}})
=
\sum_{k=1}^{2} \pi_k
\mathcal{N}(x_{\mathrm{mog}};\mu_k,\Sigma_k),
\]
with mixture weights
\[
\pi_1 = 0.8, 
\quad 
\pi_2 = 0.2,
\]
means
\[
\mu_1 =
\begin{bmatrix}
0 \\
0
\end{bmatrix},
\quad
\mu_2 =
\begin{bmatrix}
3 \\
3
\end{bmatrix},
\]
and covariance matrices
\[
\Sigma_1 =
\begin{bmatrix}
1 & 0 \\
0 & 1
\end{bmatrix},
\quad
\Sigma_2 =
\begin{bmatrix}
2 & 0 \\
0 & 2
\end{bmatrix}.
\]

For the remaining coordinates, we use an independent Gaussian tail prior:
\[
x_{\mathrm{tail}} \sim \mathcal{N}(0, I_{14}).
\]

Therefore, the full prior factorizes as
\begin{align*}
    p(x)
&=p(x_{\mathrm{mog}}) \, p(x_{\mathrm{tail}}), \\
& = \left[ 0.8 \, \mathcal{N}
\left(
\begin{bmatrix}
x_1 \\
x_2
\end{bmatrix};
\begin{bmatrix}
0 \\
0
\end{bmatrix},
\begin{bmatrix}
1 & 0 \\ 0 & 1
\end{bmatrix}
\right) + 0.2 \, \mathcal{N}
\left(
\begin{bmatrix}
x_1 \\
x_2
\end{bmatrix};
\begin{bmatrix}
3 \\ 3
\end{bmatrix},
\begin{bmatrix}
2 & 0 \\ 0 & 2
\end{bmatrix}
\right) \right] \prod_{j=3}^{16} \mathcal{N}(x_j;0,1).
\end{align*}

\paragraph{Forward Operator and Noise Model}

We consider four experimental settings formed by combining two forward operators and two noise models:

\subparagraph{Forward Operators}

\textbf{(A1) many-weak-observation.}  
The forward operator \( A \in \mathbb{R}^{14 \times 16} \) observes most coordinates of \(x\) with individual scaling:
$y_i = s_i \, x_{\mathcal{I}_i}, \quad i = 1, \dots, 14,$
where$\mathcal{I} = \{ 3,4 \dots, 16\},
\quad
s = (0.1, 0.15, 0.2, \dots, 0.75)^\top.$
The forward matrix \(A \in \mathbb{R}^{14 \times 16}\) is defined by
\[
A_{i,j} =
\begin{cases}
s_i, & \text{if } j = \mathcal{I}_i, \\
0, & \text{otherwise}.
\end{cases}
\]


\textbf{(A2) few-strong-observation.}  
The forward operator \( A \in \mathbb{R}^{4 \times 16} \) observes only a small subset of coordinates with uniform scaling:
$y_i = 3 \, x_i, \quad i = 1, \dots, 4.$
The matrix \(A \in \mathbb{R}^{4 \times 16}\) is given by
\[
A_{i,j} =
\begin{cases}
3, & \text{if } j = i,\; i \in \{1,2,3,4\}, \\
0, & \text{otherwise}.
\end{cases}
\]

\subparagraph{Noise Models}
$y = Ax +\epsilon, \quad \epsilon \sim N(0, \sigma^2)$
\textbf{(N1).}   $\sigma = 0.5$
\textbf{(N2).}   $\sigma = 0.2$

\subparagraph{Ground Truth $x$:}
\[
x_{\mathrm{true}}
=
[\,3.0,\;2.0,\;0.5,\;-0.5,\;1.0,\;0.8,\;-0.6,\;1.1,\;-0.9,\;0.4,\;0.0,\;-1.2,\;0.7,\;-0.3,\;0.5,\;-0.8\,]^{\top}.
\]
\begin{table*}[t]
\caption{RMSE, \ksd{} using analytical posterior score (An-KSD), and score-KSD using approximate posterior score (Ap-KSD) under different noise levels using sample size $N=100$ with many weak observations. Results are reported as mean and standard deviation across 5 noise draws generated from $N\sim (0,\sigma^2)$}
\centering
\resizebox{\textwidth}{!}{
\setlength{\tabcolsep}{4pt}
\begin{tabular}{l|ccc|ccc}
\toprule
& \multicolumn{3}{c|}{$\sigma=0.2$} 
& \multicolumn{3}{c}{$\sigma=0.5$} \\
\cmidrule(lr){2-4} \cmidrule(lr){5-7}
Method 
& RMSE & An-KSD & Ap-KSD
& RMSE & An-KSD & Ap-KSD \\
\midrule
DAPS\cite{zhang2024improvingdiffusioninverseproblemDAPS}
& \textbf{0.92} (0.04) & 0.72 (0.05) & 1.02 (0.08)
& 1.77 (0.30) & 1.35 (0.27) & 1.78 (0.35) \\

DDNM\cite{wang2022zeroDDNM}
& 1.00 (0.11) & 1.18 (0.13) & 1.54 (0.18)
& 2.22 (0.49) & 2.56 (0.52) & 3.19 (0.65) \\

DDRM\cite{kawar2022denoisingDDRM}
& \textbf{0.92} (0.04) & 0.67 (0.06) & 0.96 (0.09)
& 1.30 (0.13) & 0.70 (0.13) & 1.03 (0.18) \\

DiffPIR\cite{zhu2023denoisingDiffPIR}
& 1.06 (0.05) & 0.61 (0.06) & 0.88 (0.09)
& 1.92 (0.33) & 1.23 (0.28) & 1.62 (0.34) \\

DPS\cite{chung2022diffusionDPS}
& 0.95 (0.01) & \textbf{0.42} (0.02) & \textbf{0.62} (0.02)
& \textbf{1.17} (0.06) & 0.38 (0.07) & 0.57 (0.09) \\

FPS (N run)\cite{dou2024fpsdiffusion}
& 1.06 (0.09) & 1.00 (0.14) & 1.36 (0.19)
& 2.04 (0.39) & 1.52 (0.32) & 1.97 (0.39) \\

FPS (one run)
& 1.11 (0.25) & 1.69 (0.17) & 2.15 (0.23)
& 1.97 (0.38) & 2.34 (0.41) & 2.92 (0.52) \\

MCG-DIFF (N run)\cite{cardoso2023monteMCGdiff}
& 0.99 (0.02) & \textbf{0.42} (0.02) & 0.63 (0.03)
& 1.21 (0.03) & \textbf{0.24} (0.01) & \textbf{0.39} (0.03) \\

MCG-DIFF (one run)
& 1.11 (0.08) & 0.90 (0.17) & 1.18 (0.19)
& 1.27 (0.18) & 0.72 (0.12) & 0.95 (0.14) \\

PnPDM\cite{wu2024PnPDM}
& 1.00 (0.09) & 1.03 (0.10) & 1.39 (0.14)
& 2.20 (0.48) & 1.98 (0.45) & 2.51 (0.56) \\

RED-DIFF\cite{mardani2023reddiff1,song2023pseudoinverseReddiff2}
& 0.99 (0.05) & 1.61 (0.07) & 2.00 (0.12)
& 1.91 (0.35) & 2.49 (0.43) & 3.11 (0.56) \\

Posterior reference
& 1.08 (0.01) & \textbf{0.36} (0.02) & \textbf{0.54} (0.03)
& 1.29 (0.03) & \textbf{0.24} (0.00) & \textbf{0.31} (0.00) \\
\bottomrule
\end{tabular}
}
\label{tab:num_results_n100}
\end{table*}

\begin{table*}[t]
\caption{RMSE, \ksd{} using analytical posterior score (An-KSD), and score-KSD using approximate posterior score (Ap-KSD) under different noise levels using sample size $N=500$ with few strong observations. Results are reported as mean and standard deviation across 5 noise draws generated from $N\sim (0,\sigma^2)$.}
\centering
\resizebox{\textwidth}{!}{
\setlength{\tabcolsep}{4pt}
\begin{tabular}{l|ccc|ccc}
\toprule
& \multicolumn{3}{c|}{$\sigma=0.2$}
& \multicolumn{3}{c}{$\sigma=0.5$} \\
\cmidrule(lr){2-4} \cmidrule(lr){5-7}
Method
& RMSE & An-KSD & Ap-KSD
& RMSE & An-KSD & Ap-KSD \\
\midrule

DAPS\cite{zhang2024improvingdiffusioninverseproblemDAPS}
& 1.10 (0.00) & 0.60 (0.00) & 0.59 (0.00)
& 1.11 (0.00) & 0.48 (0.00) & 0.47 (0.00) \\

DDNM\cite{wang2022zeroDDNM}
& 0.89 (0.01) & 1.46 (0.02) & 1.46 (0.02)
& 0.82 (0.01) & 0.89 (0.02) & 0.90 (0.01) \\

DDRM\cite{kawar2022denoisingDDRM}
& 1.05 (0.00) & 0.73 (0.01) & 0.73 (0.00)
& 1.06 (0.01) & 0.44 (0.00) & 0.44 (0.00) \\

DiffPIR\cite{zhu2023denoisingDiffPIR}
& 1.24 (0.00) & 0.67 (0.00) & 0.68 (0.00)
& 1.25 (0.00) & 0.44 (0.00) & 0.44 (0.00) \\

DPS\cite{chung2022diffusionDPS}
& 1.17 (0.00) & \textbf{0.50} (0.00) & \textbf{0.50} (0.00)
& 1.17 (0.01) & \textbf{0.44} (0.01) & \textbf{0.44} (0.01) \\

FPS (N run)\cite{dou2024fpsdiffusion}
& 0.95 (0.01) & 2.02 (0.11) & 2.02 (0.10)
& 0.98 (0.01) & 0.71 (0.07) & 0.72 (0.08) \\

FPS (one run)
& 1.02 (0.08) & 9.65 (0.57) & 9.66 (0.57)
& 1.01 (0.06) & 1.96 (0.08) & 1.99 (0.08) \\

MCG-DIFF (N run)\cite{cardoso2023monteMCGdiff}
& 1.10 (0.01) & 1.19 (0.05) & 1.18 (0.05)
& 1.13 (0.02) & \textbf{0.44} (0.01) & \textbf{0.44} (0.00) \\

MCG-DIFF (one run)
& 1.17 (0.12) & 13.08 (4.08) & 13.10 (4.09)
& 1.12 (0.08) & 1.72 (0.32) & 1.71 (0.31) \\

PnPDM\cite{wu2024PnPDM}
& \textbf{0.96} (0.00) & 0.68 (0.00) & 0.67 (0.00)
& \textbf{0.97} (0.01) & 0.52 (0.01) & 0.50 (0.01) \\

RED-DIFF\cite{mardani2023reddiff1,song2023pseudoinverseReddiff2}
& 0.85 (0.01) & 1.16 (0.00) & 1.13 (0.00)
& 0.86 (0.01) & 1.18 (0.01) & 1.15 (0.01) \\

Posterior reference
& 1.17 (0.01) & 1.31 (0.00) & 1.31 (0.00)
& 1.20 (0.01) & 0.64 (0.00) & 0.64 (0.00) \\

\bottomrule
\end{tabular}
}
\label{tab:toy_lowd_n500}
\end{table*}

%% file: supp-exp.tex
\section{Additional Real Data Experiment}
\subsection{Additional Results}
\begin{figure}[h]
    \centering
    \includegraphics[width=0.95\linewidth]{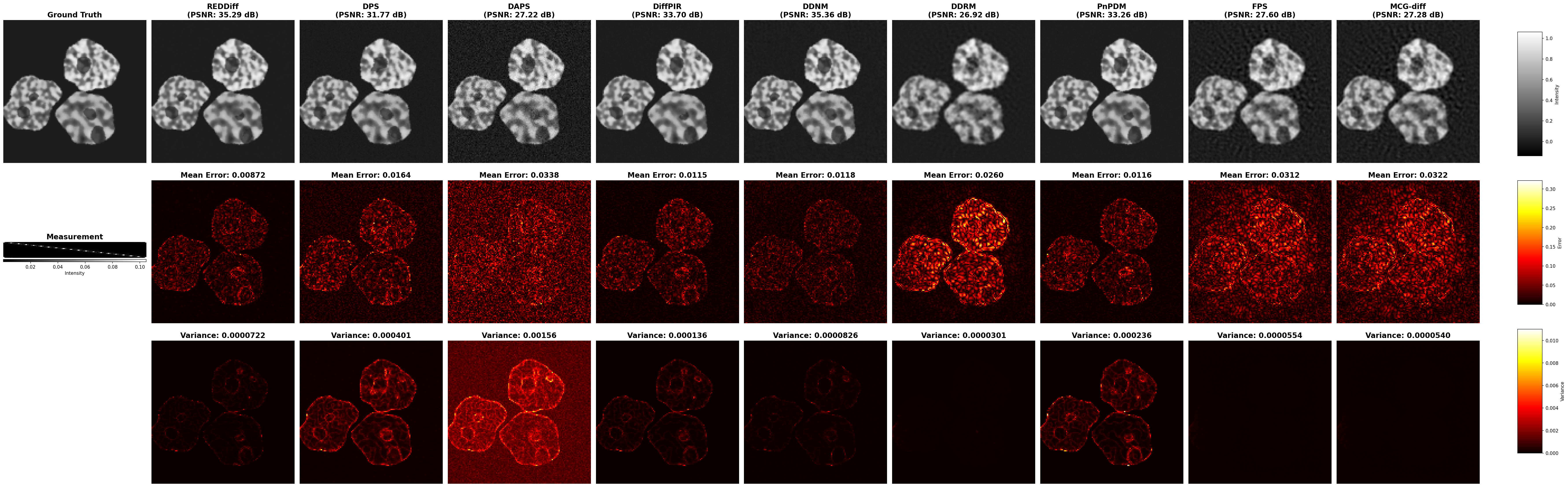}
    \caption{Inverse linear scattering with 180 receivers}
    \label{fig:inv-180}
\end{figure}

\begin{figure}[h]
    \centering
    \includegraphics[width=0.95\linewidth]{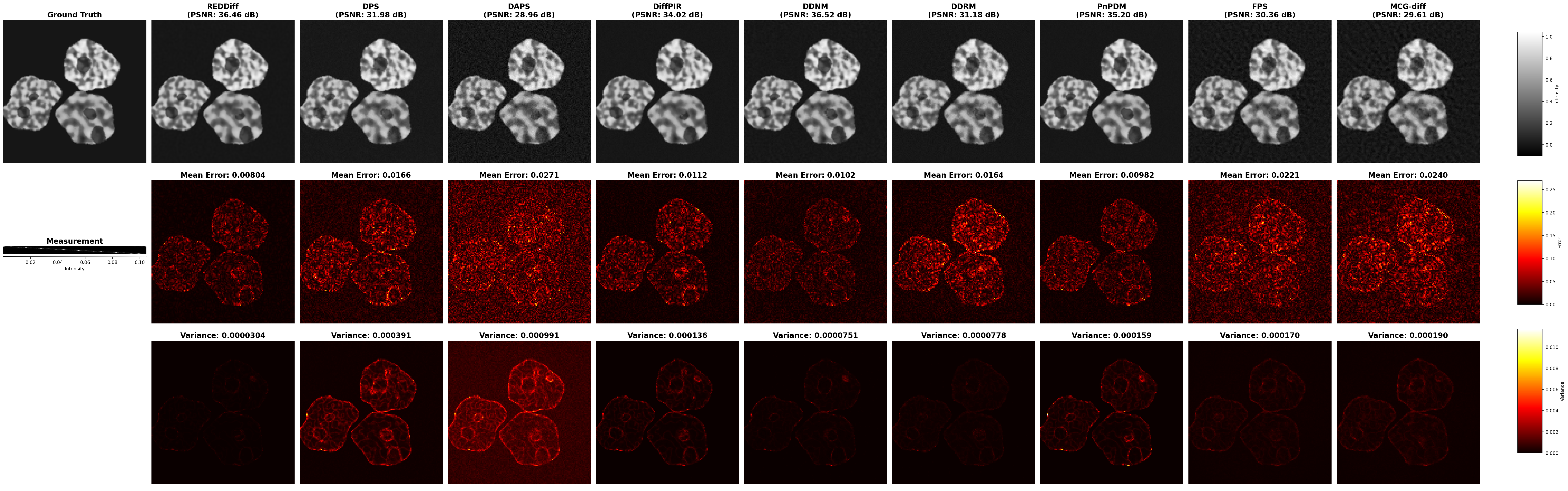}
    \caption{Inverse linear scattering with 360 receivers}
    \label{fig:inv-360}
\end{figure}

\begin{figure}[h]
    \centering
    \includegraphics[width=0.9\linewidth]{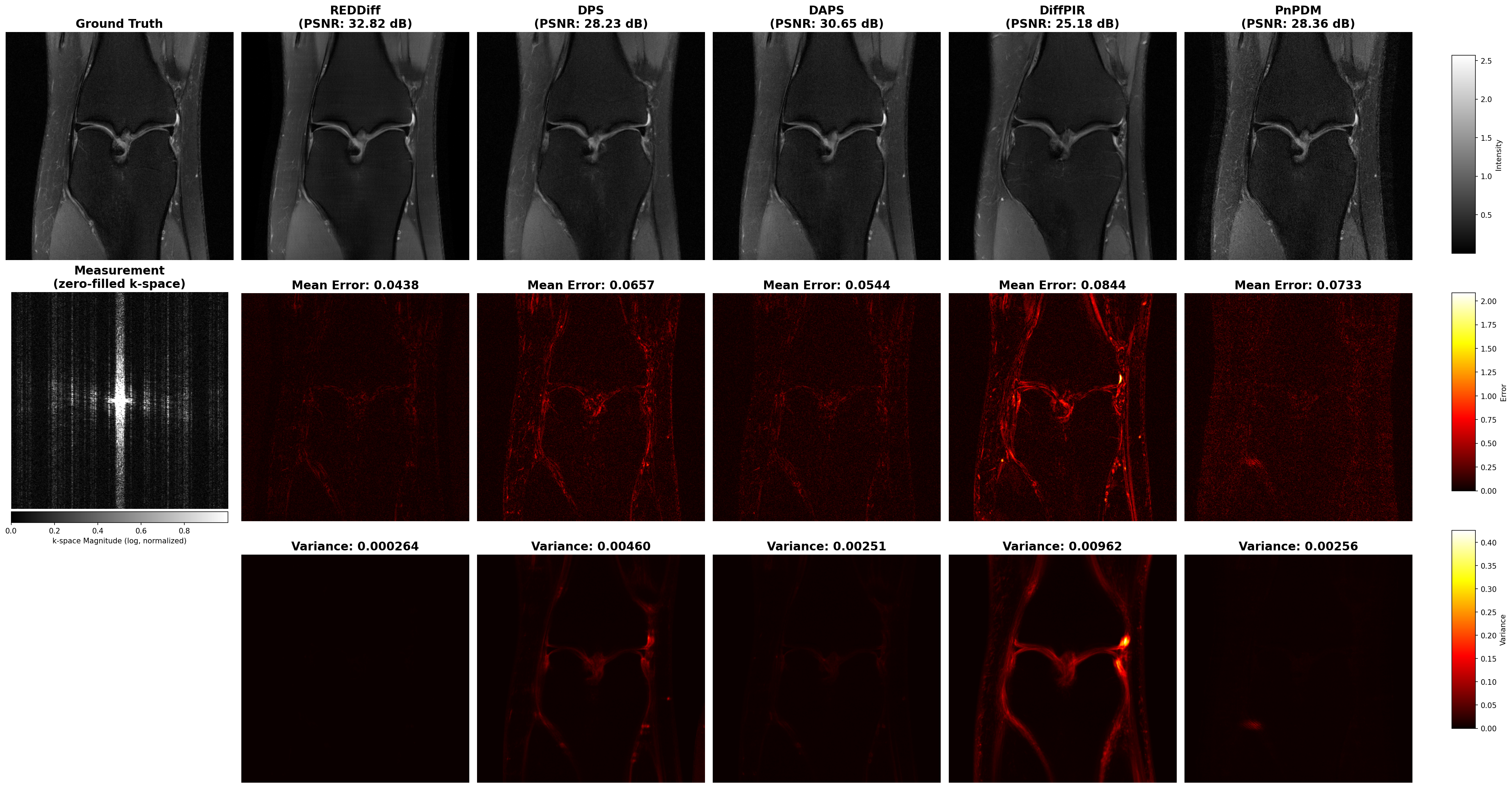}
    \caption{{Sparse-sampling MRI under $\times8$ acceleration rate (AR=8).}}
    \label{fig:mri-ar8}
\end{figure}

\begin{figure}[h]
    \centering
    \includegraphics[width=0.9\linewidth]{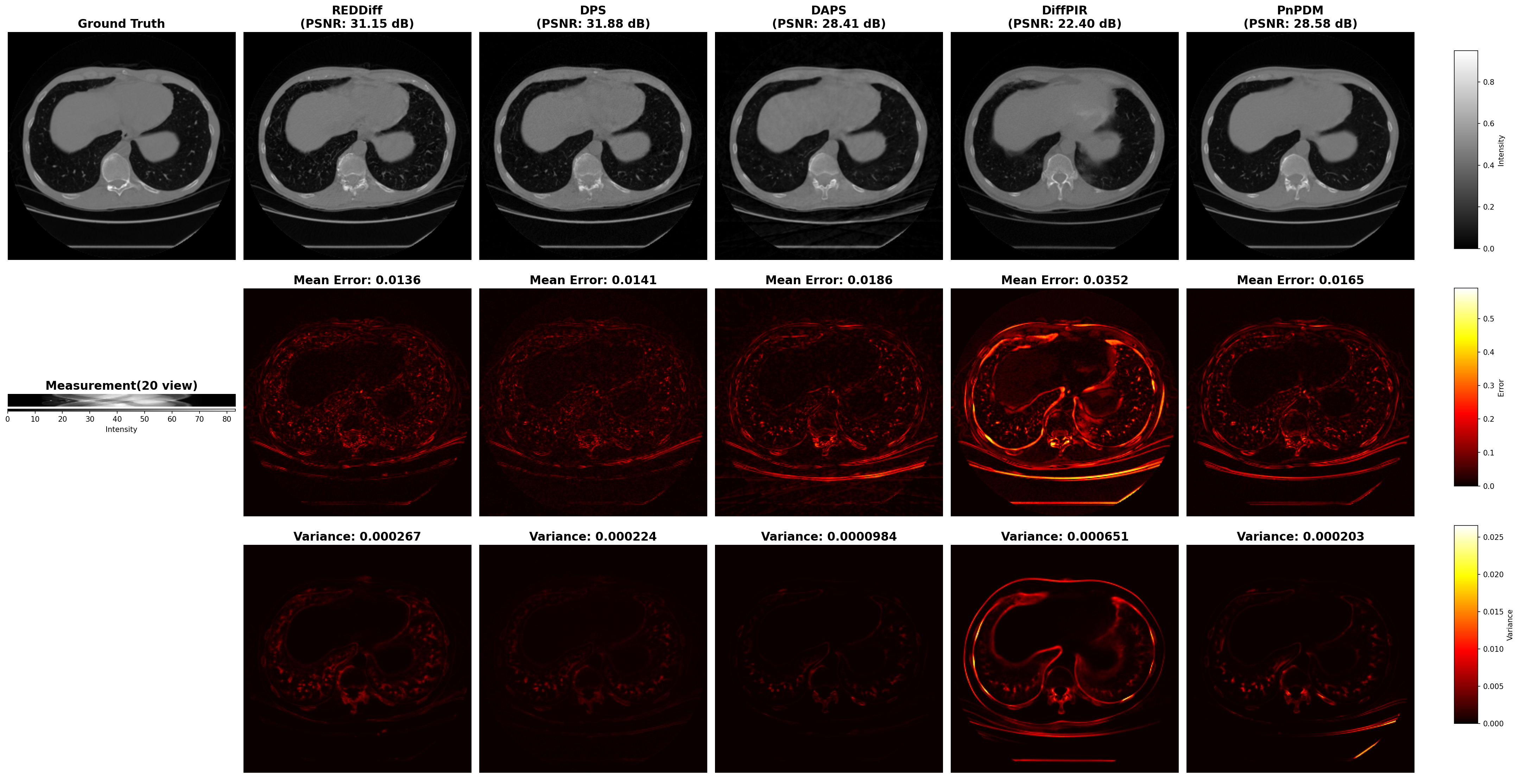}
    \caption{{20 view CT reconstruction (InD)}}
    \label{fig:ct20}
\end{figure}

\begin{figure}[h]
    \centering
    \includegraphics[width=0.9\linewidth]{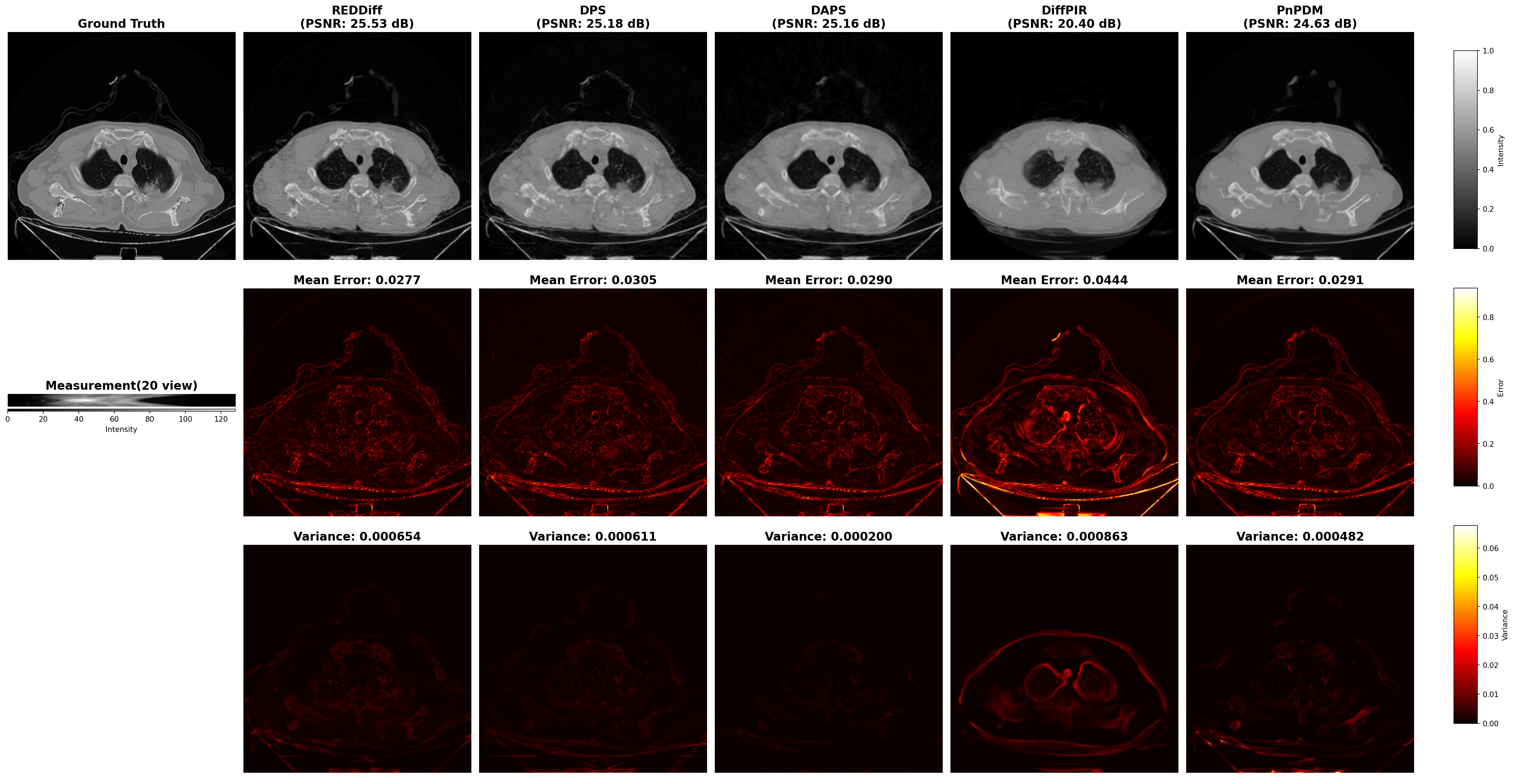}
    \caption{{20 view cancer CT reconstruction (OOD)}}
    \label{fig:ctcancer20}
\end{figure}

\begin{figure}[h]
    \centering
    \includegraphics[width=0.9\linewidth]{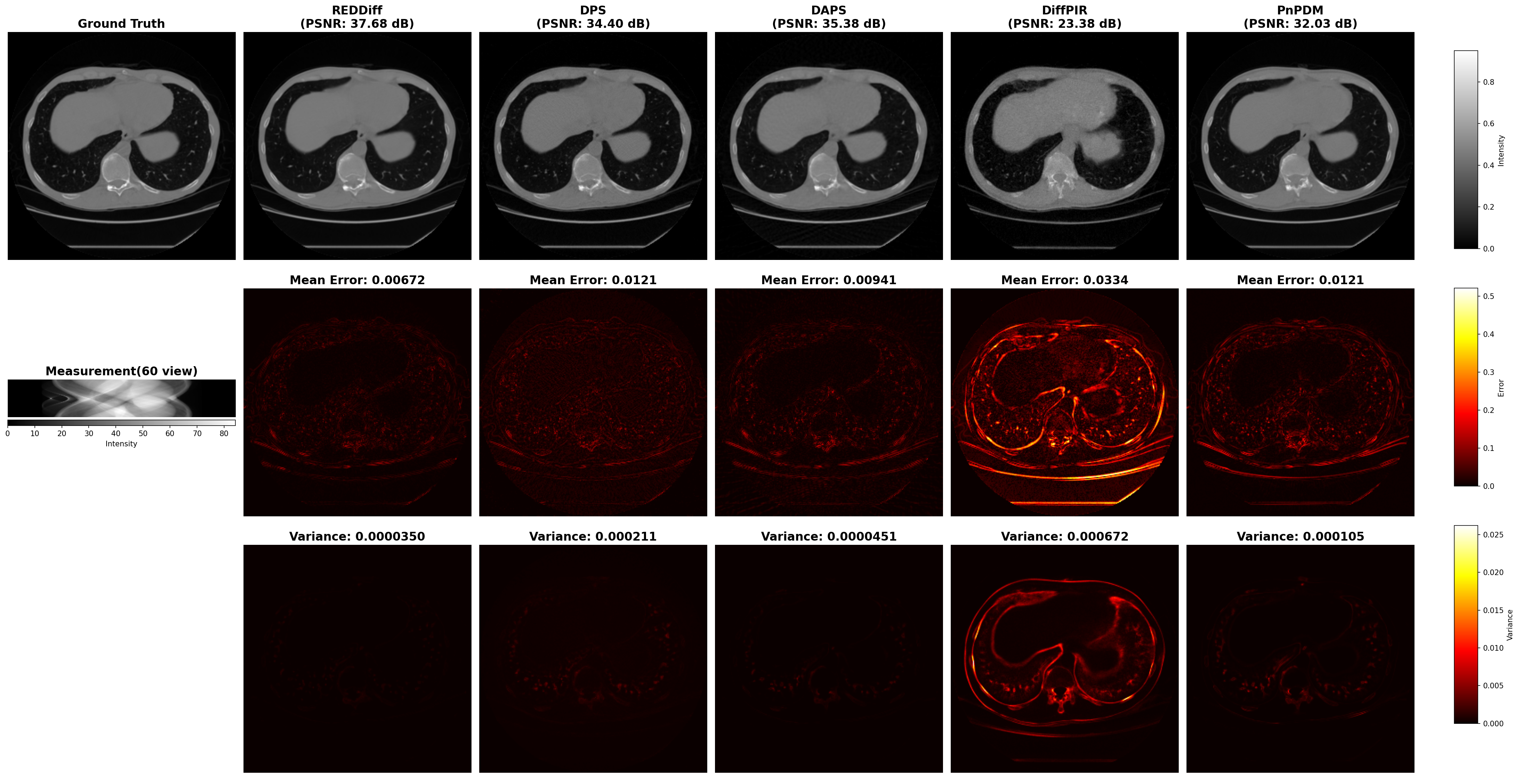}
    \caption{{60 view CT reconstruction (InD)}}
    \label{fig:ct60}
\end{figure}

\begin{table*}[t]
\centering
\resizebox{\textwidth}{!}{
\begin{tabular}{cl|cc|cc|cc|cc|cc}
\hline
\textbf{Setting} & \textbf{Method}
& \multicolumn{2}{c|}{\textbf{Test Image 1}}
& \multicolumn{2}{c|}{\textbf{Test Image 2}}
& \multicolumn{2}{c|}{\textbf{Test Image 3}}
& \multicolumn{2}{c|}{\textbf{Test Image 4}}
& \multicolumn{2}{c}{\textbf{Test Image 5}} \\
\cline{3-12}
& & PSNR(std) & KSD & PSNR(std) & KSD & PSNR(std) & KSD & PSNR(std) & KSD & PSNR(std) & KSD \\
\hline

\multirow{5}{*}{$\sigma=0.01$}
& DAPS    & 28.24(0.06)& 1075.13 & 27.64(0.08) & 998.91  & 27.53(0.09) & 1033.83 & 27.51(0.08) & 968.10  & 27.2(0.07) & 960.91\\
& DPS     & 31.54(0.2)& 221240.34 & 31.14(0.22) & 214006.06 & 30.65(0.21) & 214636.31 & 30.64(0.23) & 213411.11 & 30.47(0.2) & 213461.80\\
& DiffPIR & 22.02(0.17) & 290340.77 & 21.92(0.19) & 219286.31 & 22.17(0.21) & 229036.86 & 21.78(0.21) & 215080.82 & 21.25(0.21) & 205437.87\\
& PnPDM   & 27.4(0.48)& 40381.45  & 27.11(0.54) & 38866.05  & 27.09(0.52) & 39297.84  & 27.11(0.64) & 37512.72  & 27.07(0.54) & 35247.54 \\
& REDDiff & 31.29(0.09)& 8348.56   & 30.9(0.11) & 8223.45   & 30.57(0.1) & 8265.44   &30.39(0.1)& 7958.68   & 30.74(0.12) & 7967.16 \\
\hline
\multirow{5}{*}{$\sigma=0.1$}
&DAPS    & 28.23(0.06)& 11.01 & 27.66(0.07) & 10.19 & 27.52(0.08) & 10.57 & 27.5(0.08) & 9.94  & 27.2(0.07) & 9.80   \\
&DPS     & 31.52(0.19)& 2211.65 & 31.09(0.22) & 2138.74 & 30.61(0.2) & 2136.88 & 30.62(0.23) & 2133.68 &30.36(0.21)& 2135.96  \\
&DiffPIR & 22.02(0.17)& 2903.88 & 21.92(0.19) & 2193.25 & 22.17(0.21) & 2290.67 & 21.78(0.21)& 2149.57 & 21.07(0.22) & 2119.35  \\
&PnPDM   & 27.4(0.48)& 404.25  & 27.09(0.53) & 389.15  & 27.09(0.52) & 393.21  & 27.11(0.64) & 375.60  & 26.97(0.55) & 357.95  \\
&REDDiff & 31.31(0.09)& 83.97   & 30.87(0.11) & 82.74   & 30.56(0.09) & 83.16   & 30.38(0.1) & 80.23   & 30.58(0.1) & 79.35    \\
\hline
\multirow{5}{*}{$\sigma\approx0.26$}
& DAPS    & 28.21(0.06)& 2.36 & 27.66(0.07) & 2.06 &27.49(0.08)& 2.03 & 27.47(0.08) & 2.04  & 27.18(0.07) & 2.07 \\
& DPS     &31.42(0.19) & 337.91 & 30.91(0.22) & 263.07 &30.55(0.2)& 245.41 & 30.49(0.23)& 258.79 & 30.43(0.2) & 275.05  \\
& DiffPIR & 22.02(0.17)& 440.86 & 21.93(0.19) & 270.35 & 22.17(0.21) & 263.98 & 21.78(0.21)&261.56 & 21.07(0.22) & 272.78  \\
& PnPDM   & 27.4(0.48)& 61.57  &27.05(0.53)& 48.18  & 27.08(0.52) & 45.50  & 27.13(0.63) & 45.78  & 26.97(0.54) & 46.31   \\
& REDDiff & 31.31(0.1)& 13.26  & 30.78(0.11) & 10.74  & 30.52(0.09) & 10.10  &30.42(0.1) & 10.34  &30.66(0.1)& 10.78   \\

\hline
\end{tabular}
}
\caption{PSNR and KSD scores across five test images on 20-view CT reconstruction (InD) under different degradation settings.}
\label{tab:combined_ksd_ct}
\end{table*}

\begin{table*}[t]
\centering
\resizebox{\textwidth}{!}{
\begin{tabular}{cl|cc|cc|cc|cc|cc}
\hline
\textbf{Setting} & \textbf{Method}
& \multicolumn{2}{c|}{\textbf{Test Image 1}}
& \multicolumn{2}{c|}{\textbf{Test Image  2}}
& \multicolumn{2}{c|}{\textbf{Test Image  3}}
& \multicolumn{2}{c|}{\textbf{Test Image  4}}
& \multicolumn{2}{c}{\textbf{Test Image  5}} \\
\cline{3-12}
& & PSNR(std)& KSD & PSNR(std) & KSD & PSNR(std) & KSD & PSNR(std) & KSD & PSNR(std) & KSD \\
\hline

\multirow{5}{*}{CT cancer, $\sigma=0.01$}
& DAPS   & 25.04(0.06) & 1740.52 & 25.22(0.06) & 1743.21 & 24.9(0.06) & 1734.35 & 27.64(3.37) & 1206.38 & 25.06(0.06) & 1730.23  \\
& DPS     & 24.8(0.2)& 210922.06 & 24.92(0.16)& 209845.06 & 24.97(0.21) & 210047.86 & 28.46(0.4) & 213748.52 & 25.25(0.22) & 210133.82   \\
& DiffPIR &20.2(0.09) & 340760.73 & 20.21(0.1) & 336048.86 & 19.75(0.11) & 341703.17 &21.54(0.13)& 342059.94 & 20.16(0.1) & 348602.84  \\
& PnPDM   & 24.36(0.14)& 38180.62  & 24.46(0.23) & 38745.92  &24.38(0.24)& 40020.66  &26.95(0.27)& 40112.75  &24.43(0.18)& 41974.18   \\
& REDDiff & 25.36(0.07)& 11093.03  & 25.53(0.08) & 11063.26  & 25.63(0.1) & 11199.65  & 28.01(0.09) & 10325.06  & 25.64(0.06) & 11595.82  \\

\hline
\multirow{5}{*}{CT cancer, $\sigma\approx0.1$}
& DAPS    & 25.03(0.06) & 17.44   & 25.21(0.06) & 17.48   & 24.90(0.06) & 17.40   & 27.63(3.37) & 12.67   & 25.06(0.06) & 17.34 \\
& DPS     & 24.79(0.21) & 2108.72 & 24.91(0.16) & 2097.73 & 24.92(0.21) & 2100.35 & 28.46(0.39) & 2126.26 & 25.25(0.22) & 2097.72 \\
& DiffPIR & 20.20(0.09) & 3407.80 & 20.21(0.10) & 3360.23 & 19.75(0.11) & 3417.05 & 21.54(0.13) & 3420.62 & 20.16(0.10) & 3485.87 \\
& PnPDM   & 24.36(0.14) & 382.65  & 24.45(0.23) & 387.95  & 24.38(0.23) & 401.35  & 26.95(0.27) & 401.75  & 24.43(0.18) & 420.58 \\
& REDDiff & 25.36(0.07) & 111.31  & 25.53(0.08) & 111.12  & 25.63(0.10) & 112.35  & 28.01(0.09) & 103.83  & 25.64(0.06) & 116.41 \\

\hline
\multirow{5}{*}{CT cancer, $\sigma\approx0.26$}
& DAPS    & 25.01(0.06)& 2.69 & 25.19(0.06) & 2.73 & 24.88(0.06) & 2.72 & 27.6(3.37) & 3.96 & 25.04(0.06) & 2.73   \\
& DPS     & 24.77(0.2)& 269.03 &24.88(0.16) & 270.46 & 24.75(0.2)& 269.96 & 28.41(0.4) & 275.79 & 25.21(0.22) & 271.99  \\
& DiffPIR & 20.2(0.09)& 433.71 & 20.21(0.1) & 433.67 & 19.75(0.11) & 439.82 & 21.54(0.13) & 440.40 &20.17(0.1) & 452.78   \\
& PnPDM   & 24.35(0.14)& 48.82  & 24.43(0.23)& 50.30  & 24.37(0.23) & 51.93  & 26.94(0.27) & 51.98  & 24.42(0.19) & 54.83    \\
& REDDiff & 25.33(0.07)& 14.50  & 25.51(0.08) & 14.71  & 25.6(0.11) & 14.80  & 27.98(0.08) & 13.80  & 25.63(0.06) & 15.46    \\

\hline
\end{tabular}
}
\caption{PSNR and KSD scores across five test images on 20-view cancer CT reconstruction (OOD) under different degradation settings.}
\label{tab:combined_ksd_ct_cancer}
\end{table*}

\FloatBarrier
\subsection{Hyperparameter Details}
\label{supp:hyper_real}

For tasks (i) linear inverse scattering (180 and 360 views respectively) and (ii) MRI (simulated), we adopt the solver hyperparameters from Table~12 of InverseBench~\citep{zheng2025inversebench}  except for the noise level in DAPS for MRI task, where we fail to make reasonable reconstruction, and made a sweeping based on best accuracy, then changed the noise level into 0.008. For the CT task, hyperparameters are tuned separately following standard validation procedures as stated in InverseBench~\citep{zheng2025inversebench} Section B.7.2. All hyperparameters are reported in Table~\ref{tab:hyperparameter_setting}.

\begin{table*}[h]
\centering
\resizebox{\textwidth}{!}{
\begin{tabular}{llccc}
\toprule
\textbf{Methods/Parameters} 
& \textbf{Search space}
& \textbf{Linear inverse scattering (360 / 180)}
& \textbf{MRI (Sim.)}
& \textbf{CT (20 view / 60 view)}\\
\midrule

\textbf{DPS} & & & & \\
Guidance scale & $[10^{-3}, 10^3]$ & $280 / 380$ & $0.589$ & $10^{-1}$ \\
\midrule

\textbf{REDDiff} & & & & \\
Learning rate & $[10^{-4}, 1.0]$ & $0.04$ & $4 \times 10^{-2}$ & $0.2$ \\
Regularization $\lambda_{\mathrm{base}}$ & $[10^{-3}, 1.0]$ & $0.0005$ & $2.33 \times 10^{-1}$ & $20$ \\
Regularization schedule & constant, linear, sqrt & constant & sqrt & constant \\
Gradient weight & $[10^{-2}, 10^2]$ & $1500$ & $6.68 \times 10^1$ & $0.5$  \\
\midrule

\textbf{DiffPIR} & & & & \\
\# sampling steps & $\{200,400,\ldots,1000\}$ & $200$ & $1000$ & $50$  \\
Regularization $\lambda$ & $[1,10^5]$ & $4 \times 10^{-4}$ & $163$ & $1000 / 2000$  \\
Stochasticity $\zeta$ & $[10^{-5},1]$ & $1$ & $0.114$ & $0.5 / 0.4$  \\
\midrule

\textbf{PnPDM} & & & & \\
Annealing step & $[50,200]$ & $100$ & $100$ & $100$  \\
Annealing sigma max & $[10,50]$ & $10$ & $10$ & $50$  \\
Annealing decay rate & $[0.60,0.99]$ & $0.9$ & $0.93$ & $0.93$  \\
Langevin step size & $[10^{-6},10^{-3}]$ & $2 \times 10^{-5}$ & $10^{-6}$ & $1 \times 10^{-4}$  \\
Langevin step number & $[10,500]$ & $200$ & $200$ & $100$ \\
\midrule

\textbf{DAPS} & & & &  \\
Annealing step & $[50,200]$ & $200$ & $200$ & $100$  \\
Diffusion step & $[1,10]$ & $10$ & $5$ & $5$  \\
Langevin step size & $[10^{-6},10^{-3}]$ & $4 \times 10^{-5}$ & $1.03 \times 10^{-5}$ & $1.8 \times 10^{-7}$  \\
Langevin step number & $[10,500]$ & $50$ & $100$ & $100$  \\
Noise level & $[10^{-4},10]$ & $10^{-4}$ & $0.008$ & $0.015$  \\
\midrule
\textbf{DDRM} & & & & \\
Stochasticity $\eta$ & $[0,1]$ & $0.85$ & -- & -- \\
\midrule
\textbf{DDNM} & & & & \\
Stochasticity $\eta$ & $[0,1]$ & $0.95$ & -- & -- \\
\# time-travel steps $L$ & $[0,5]$ & $1$ & -- & -- \\

\midrule
\textbf{IIGDM} & & & & \\
Stochasticity $\eta$ & $[0,1]$ & $0.2$ & -- & -- \\

\midrule
\textbf{FPS} & & & & \\
Stochasticity $\eta$ & $[0,1]$ & $0.9$ & -- & -- \\
\# particles & $[1,20]$ & $20$ & -- & -- \\

\midrule
\textbf{MCG-diff} & & & & \\
\# particles & $[1,64]$ & $16$ & -- & -- \\
\bottomrule
\end{tabular}
}
\caption{Hyperparameter settings used for each inverse problem.}
\label{tab:hyperparameter_setting}
\end{table*}

\begin{table}[h]
\centering
\begin{tabular}{ccccc}
\hline
\textbf{Guidance scale} & \textbf{PSNR} & \textbf{KSD} \\
\hline
0.2 & 30.81(0.42) & 458615.46 \\
0.4 & 27.86(1.90) & 1063583.36 \\
0.6 & 26.17(3.22) & 1584447.02 \\
0.8 & 24.11(3.85) & 2137034.27 \\
1.0 & 21.05(3.94) & 3227196.72 \\
\hline
\end{tabular}
\caption{Hyperparameter sensitivity evaluation on DPS solving 20-view CT reconstruction degraded with $\sigma=0.01$}
\label{tab:dps-hyper}
\end{table}
\subsection{Computation Resources}
\paragraph{Diffusion Prior Training.}
Training of the CT diffusion prior model was performed using a single NVIDIA A100 GPU for approximately two days. 

\paragraph{Inference and Sampling.}
Inference and sampling experiments were conducted using a combination of NVIDIA L40S GPUs and A100 GPUs. Each inference takes 1-5 minutes depends on different methods.

\paragraph{KSD Experiments.}
Kernel Stein Discrepancy (KSD) evaluation used L40S GPUs.

%% file: supp-proof.tex
\FloatBarrier
\newpage
\section{Proof}
\begin{proof}[Proof of Proposition ~ \ref{prop:stein identity}]
By definition,
\[
\begin{aligned}
    \mathbb E_{X\sim p(x\mid y_0)}
    \left[
        \mathcal T_p f(X)
    \right]
    &=
    \int
    \left[
        s_p(x)^\top f(x)
        +
        \nabla_x\cdot f(x)
    \right]
    p(x\mid y_0)\,dx.
\end{aligned}
\]
Since
\[
    s_p(x)
    =
    \nabla_x \log p(x\mid y_0)
    =
    \frac{\nabla_x p(x\mid y_0)}{p(x\mid y_0)},
\]
we have
\[
    s_p(x)p(x\mid y_0)
    =
    \nabla_x p(x\mid y_0).
\]
Therefore,
\[
\begin{aligned}
    \mathbb E_{p}
    \left[
        \mathcal T_p f(X)
    \right]
    &=
    \int
    f(x)^\top \nabla_x p(x\mid y_0)\,dx
    +
    \int
    p(x\mid y_0)\nabla_x\cdot f(x)\,dx \\
    &=
    \int
    \nabla_x\cdot
    \left[
        p(x\mid y_0) f(x)
    \right] dx.
\end{aligned}
\]
By the divergence theorem and the boundary condition,
\[
    \int
    \nabla_x\cdot
    \left[
        p(x\mid y_0) f(x)
    \right] dx
    =
    0.
\]
Hence,
\[
    \mathbb E_{X\sim p(x\mid y_0)}
    \left[
        \mathcal T_p f(X)
    \right]
    =
    0.
\]
\end{proof}

\begin{proof}[Proof of Proposition ~\ref{prop:valid measure}]
We first show non-negativity. By definition,
\[
    \mathrm{KSD}(q,p)
    =
    \sup_{\|f\|_{\mathcal H^d}\le 1}
    \mathbb E_{X\sim q(x\mid y_0)}
    \left[
        \mathcal T_p f(X)
    \right].
\]
Since the zero function \(f\equiv 0\) belongs to \(\mathcal H^d\), we have
\[
    \mathbb E_{X\sim q(x\mid y_0)}
    \left[
        \mathcal T_p f(X)
    \right]
    =
    0,
\]
which implies
\[
    \mathrm{KSD}(q,p)\ge 0.
\]

Next, suppose \(q(x\mid y_0)=p(x\mid y_0)\). By the Stein identity,
\[
    \mathbb E_{X\sim p(x\mid y_0)}
    \left[
        \mathcal T_p f(X)
    \right]
    =
    0
\]
for all admissible test functions \(f\). Therefore,
\[
    \mathrm{KSD}(q,p)
    =
    \sup_{\|f\|_{\mathcal H^d}\le 1}
    0
    =
    0.
\]

Conversely, suppose \(\mathrm{KSD}(q,p)=0\). Then
\[
    \mathbb E_{X\sim q(x\mid y_0)}
    \left[
        \mathcal T_p f(X)
    \right]
    =
    0
    \qquad
    \forall f\in\mathcal H^d.
\]
Under standard regularity conditions on \(p(x\mid y_0)\) and for a
characteristic kernel \(k\) (e.g., the inverse multiquadric kernel), the only
distribution satisfying these Stein identities is \(p(x\mid y_0)\). Hence,
\[
    q(x\mid y_0)=p(x\mid y_0).
\]

\end{proof}

\begin{proof}[Proof of Proposition ~\ref{prop:closed form}]
Since \(\hat q_N\) is empirical,
\[
    \mathbb E_{X\sim \hat q_N}
    \left[
        \mathcal T_p f(X)
    \right]
    =
    \frac{1}{N}
    \sum_{i=1}^N
    \mathcal T_p f(x_i).
\]
Therefore,
\[
    \mathrm{KSD}(\hat q_N,p)
    =
    \sup_{\|f\|_{\mathcal H^d}\le 1}
    \frac{1}{N}
    \sum_{i=1}^N
    \left[
        s_p(x_i)^\top f(x_i)
        +
        \nabla_x\cdot f(x_i)
    \right].
\]

Using the reproducing property of the RKHS, the functional
\[
    f
    \mapsto
    \frac{1}{N}
    \sum_{i=1}^N
    \mathcal T_p f(x_i)
\]
can be written as an inner product in \(\mathcal H^d\):
\[
    \frac{1}{N}
    \sum_{i=1}^N
    \mathcal T_p f(x_i)
    =
    \left\langle
        f,
        \frac{1}{N}
        \sum_{i=1}^N
        \xi_p(x_i,\cdot)
    \right\rangle_{\mathcal H^d},
\]
where
\[
    \xi_p(x_i,\cdot)
    =
    s_p(x_i) k(x_i,\cdot)
    +
    \nabla_{x_i} k(x_i,\cdot).
\]

Hence, by Cauchy--Schwarz,
\[
\begin{aligned}
    \mathrm{KSD}(\hat q_N,p)
    &=
    \sup_{\|f\|_{\mathcal H^d}\le 1}
    \left\langle
        f,
        \frac{1}{N}
        \sum_{i=1}^N
        \xi_p(x_i,\cdot)
    \right\rangle_{\mathcal H^d} \\
    &=
    \left\|
        \frac{1}{N}
        \sum_{i=1}^N
        \xi_p(x_i,\cdot)
    \right\|_{\mathcal H^d}.
\end{aligned}
\]

Squaring both sides gives
\[
\begin{aligned}
    \mathrm{KSD}^2(\hat q_N,p)
    &=
    \left\langle
        \frac{1}{N}
        \sum_{i=1}^N
        \xi_p(x_i,\cdot),
        \frac{1}{N}
        \sum_{j=1}^N
        \xi_p(x_j,\cdot)
    \right\rangle_{\mathcal H^d} \\
    &=
    \frac{1}{N^2}
    \sum_{i=1}^N
    \sum_{j=1}^N
    \left\langle
        \xi_p(x_i,\cdot),
        \xi_p(x_j,\cdot)
    \right\rangle_{\mathcal H^d}.
\end{aligned}
\]

Define
\[
    u_p(x_i,x_j)
    :=
    \left\langle
        \xi_p(x_i,\cdot),
        \xi_p(x_j,\cdot)
    \right\rangle_{\mathcal H^d}.
\]
Expanding this RKHS inner product yields
\[
\begin{aligned}
    u_p(x_i,x_j)
    &=
    s_p(x_i)^\top k(x_i,x_j)s_p(x_j)
    +
    s_p(x_i)^\top \nabla_{x_j} k(x_i,x_j)
    \\
    &\quad+
    s_p(x_j)^\top \nabla_{x_i} k(x_i,x_j)
    +
    \mathrm{tr}\!\left(
        \nabla_{x_i}\nabla_{x_j} k(x_i,x_j)
    \right).
\end{aligned}
\]
This proves the closed-form estimator for \(\mathrm{KSD}^2(\hat q_N,p)\).
\end{proof}

%% file: References.bib
@article{nagel2016unified,
  title={A unified framework for multilevel uncertainty quantification in Bayesian inverse problems},
  author={Nagel, Joseph B and Sudret, Bruno},
  journal={Probabilistic Engineering Mechanics},
  volume={43},
  pages={68--84},
  year={2016},
  publisher={Elsevier}
}

@inproceedings{
zheng2025inversebench,
title={InverseBench: Benchmarking Plug-and-Play Diffusion Priors for Inverse Problems in Physical Sciences},
author={Hongkai Zheng and Wenda Chu and Bingliang Zhang and Zihui Wu and Austin Wang and Berthy Feng and Caifeng Zou and Yu Sun and Nikola Borislavov Kovachki and Zachary E Ross and Katherine Bouman and Yisong Yue},
booktitle={The Thirteenth International Conference on Learning Representations},
year={2025},
howpublished={https://openreview.net/forum?id=U3PBITXNG6}
}

@article{chen2025solvingDBP,
  title={Solving inverse problems via diffusion-based priors: An approximation-free ensemble sampling approach},
  author={Chen, Haoxuan and Ren, Yinuo and Min, Martin Renqiang and Ying, Lexing and Izzo, Zachary},
  journal={arXiv preprint arXiv:2506.03979},
  year={2025}
}

@article{cardoso2023monteMCGdiff,
  title={Monte Carlo guided diffusion for Bayesian linear inverse problems},
  author={Cardoso, Gabriel and Idrissi, Yazid Janati El and Corff, Sylvain Le and Moulines, Eric},
  journal={arXiv preprint arXiv:2308.07983},
  year={2023}
}

@article{chung2022diffusionDPS,
  title={Diffusion posterior sampling for general noisy inverse problems},
  author={Chung, Hyungjin and Kim, Jeongsol and Mccann, Michael T and Klasky, Marc L and Ye, Jong Chul},
  journal={arXiv preprint arXiv:2209.14687},
  year={2022}
}

@article{chan2024estimatingEUAU_signlemodel,
  title={Estimating epistemic and aleatoric uncertainty with a single model},
  author={Chan, Matthew and Molina, Maria and Metzler, Chris},
  journal={Advances in Neural Information Processing Systems},
  volume={37},
  pages={109845--109870},
  year={2024}
}

@article{hofman2024quantifyingAandEUproperscore,
  title={Quantifying aleatoric and epistemic uncertainty with proper scoring rules},
  author={Hofman, Paul and Sale, Yusuf and H{\"u}llermeier, Eyke},
  journal={arXiv preprint arXiv:2404.12215},
  year={2024}
}

@inproceedings{wu2024PnPDM,
    title={Principled Probabilistic Imaging using Diffusion Models as Plug-and-Play Priors},
    author={Zihui Wu and Yu Sun and Yifan Chen and Bingliang Zhang and Yisong Yue and Katherine Bouman},
    booktitle={The Thirty-eighth Annual Conference on Neural Information Processing Systems},
    year={2024},
    howpublished={https://openreview.net/forum?id=Xq9HQf7VNV}
}

@article{mardani2023reddiff1,
  title={A Variational Perspective on Solving Inverse Problems with Diffusion Models},
  author={Mardani, Morteza and Song, Jiaming and Kautz, Jan and Vahdat, Arash},
  journal={arXiv preprint arXiv:2305.04391},
  year={2023}
}

@inproceedings{song2023pseudoinverseReddiff2,
  title={Pseudoinverse-guided diffusion models for inverse problems},
  author={Song, Jiaming and Vahdat, Arash and Mardani, Morteza and Kautz, Jan},
  booktitle={International Conference on Learning Representations},
  year={2023}
}

@inproceedings{kawar2022denoisingDDRM,
    title={Denoising Diffusion Restoration Models},
    author={Bahjat Kawar and Michael Elad and Stefano Ermon and Jiaming Song},
    booktitle={Advances in Neural Information Processing Systems},
    year={2022}
}

@article{wang2022zeroDDNM,
  title={Zero-Shot Image Restoration Using Denoising Diffusion Null-Space Model},
  author={Wang, Yinhuai and Yu, Jiwen and Zhang, Jian},
  journal={The Eleventh International Conference on Learning Representations},
  year={2023}
}

@misc{kim2025distributionshiftuncertaintyestimationinverse,
      title={Towards Distribution-Shift Uncertainty Estimation for Inverse Problems with Generative Priors}, 
      author={Namhoon Kim and Sara Fridovich-Keil},
      year={2025},
      eprint={2510.10947},
      archivePrefix={arXiv},
      primaryClass={cs.CV},
      howpublished={https://arxiv.org/abs/2510.10947}, 
}

@misc{zhang2024improvingdiffusioninverseproblemDAPS,
      title={Improving Diffusion Inverse Problem Solving with Decoupled Noise Annealing}, 
      author={Bingliang Zhang and Wenda Chu and Julius Berner and Chenlin Meng and Anima Anandkumar and Yang Song},
      year={2024},
      eprint={2407.01521},
      archivePrefix={arXiv},
      primaryClass={cs.LG},
      howpublished={https://arxiv.org/abs/2407.01521}, 
}

@inproceedings{zhu2023denoisingDiffPIR,
      title={Denoising Diffusion Models for Plug-and-Play Image Restoration},
      author={Yuanzhi Zhu and Kai Zhang and Jingyun Liang and Jiezhang Cao and Bihan Wen and Radu Timofte and Luc Van Gool},
      booktitle={IEEE Conference on Computer Vision and Pattern Recognition Workshops (NTIRE)},
      year={2023},
}

@inproceedings{
dou2024fpsdiffusion,
title={Diffusion Posterior Sampling for Linear Inverse Problem Solving: A Filtering Perspective},
author={Zehao Dou and Yang Song},
booktitle={The Twelfth International Conference on Learning Representations},
year={2024},
howpublished={https://openreview.net/forum?id=tplXNcHZs1}
}

@misc{song2024Resamplesolvinginverseproblemslatent,
      title={Solving Inverse Problems with Latent Diffusion Models via Hard Data Consistency}, 
      author={Bowen Song and Soo Min Kwon and Zecheng Zhang and Xinyu Hu and Qing Qu and Liyue Shen},
      year={2024},
      eprint={2307.08123},
      archivePrefix={arXiv},
      primaryClass={cs.CV},
      howpublished={https://arxiv.org/abs/2307.08123}, 
}

@misc{daras2024surveydiffusionmodelsinverse,
      title={A Survey on Diffusion Models for Inverse Problems}, 
      author={Giannis Daras and Hyungjin Chung and Chieh-Hsin Lai and Yuki Mitsufuji and Jong Chul Ye and Peyman Milanfar and Alexandros G. Dimakis and Mauricio Delbracio},
      year={2024},
      eprint={2410.00083},
      archivePrefix={arXiv},
      primaryClass={cs.LG},
      howpublished={https://arxiv.org/abs/2410.00083}, 
}

@misc{DDPMho2020denoisingdiffusionprobabilisticmodels,
      title={Denoising Diffusion Probabilistic Models}, 
      author={Jonathan Ho and Ajay Jain and Pieter Abbeel},
      year={2020},
      eprint={2006.11239},
      archivePrefix={arXiv},
      primaryClass={cs.LG},
      howpublished={https://arxiv.org/abs/2006.11239}, 
}

@misc{song2021scorebasedgenerativemodelingstochastic,
      title={Score-Based Generative Modeling through Stochastic Differential Equations}, 
      author={Yang Song and Jascha Sohl-Dickstein and Diederik P. Kingma and Abhishek Kumar and Stefano Ermon and Ben Poole},
      year={2021},
      eprint={2011.13456},
      archivePrefix={arXiv},
      primaryClass={cs.LG},
      howpublished={https://arxiv.org/abs/2011.13456}, 
}

@misc{ddimsong2022denoisingdiffusionimplicitmodels,
      title={Denoising Diffusion Implicit Models}, 
      author={Jiaming Song and Chenlin Meng and Stefano Ermon},
      year={2022},
      eprint={2010.02502},
      archivePrefix={arXiv},
      primaryClass={cs.LG},
      howpublished={https://arxiv.org/abs/2010.02502}, 
}

@inproceedings{UncertaintyDeepLearningNIPS2017_2650d608,
 author = {Kendall, Alex and Gal, Yarin},
 booktitle = {Advances in Neural Information Processing Systems},
 editor = {I. Guyon and U. Von Luxburg and S. Bengio and H. Wallach and R. Fergus and S. Vishwanathan and R. Garnett},
 pages = {},
 publisher = {Curran Associates, Inc.},
 title = {What Uncertainties Do We Need in Bayesian Deep Learning for Computer Vision?},
 url = {https://proceedings.neurips.cc/paper_files/paper/2017/file/2650d6089a6d640c5e85b2b88265dc2b-Paper.pdf},
 volume = {30},
 year = {2017}
}

@misc{song2022solvinginverseproblemsmedical,
      title={Solving Inverse Problems in Medical Imaging with Score-Based Generative Models}, 
      author={Yang Song and Liyue Shen and Lei Xing and Stefano Ermon},
      year={2022},
      eprint={2111.08005},
      archivePrefix={arXiv},
      primaryClass={eess.IV},
      howpublished={https://arxiv.org/abs/2111.08005}, 
}

@book{inverse_astronomy_osti_5734250,
  author       = {Craig, I and Brown, J},
  title        = {Inverse problems in astronomy},
  url          = {https://www.osti.gov/biblio/5734250},
  place        = {United States},
  publisher    = {Adam Hilger Ltd.,Accord, MA},
  year         = {1985},
  month        = {12}}

@book{ocean_Wunsch_1996, place={Cambridge}, title={The Ocean Circulation Inverse Problem}, publisher={Cambridge University Press}, author={Wunsch, Carl}, year={1996}}

@article{virieux_wave10.1190/1.3238367,
    author = {Virieux, J. and Operto, S.},
    title = {An overview of full-waveform inversion in exploration geophysics},
    journal = {Geophysics},
    volume = {74},
    number = {6},
    pages = {WCC1-WCC26},
    year = {2009},
    month = {12},
    issn = {0016-8033},
    doi = {10.1190/1.3238367},
    url = {https://doi.org/10.1190/1.3238367},
    eprint = {https://pubs.geoscienceworld.org/seg/geophysics/article-pdf/74/6/WCC1/3138505/gsgpy_74_6_WCC1.pdf},
}

@article{audio_signal_lemercier2025diffusion,
  title={Diffusion models for audio restoration: A review [special issue on model-based and data-driven audio signal processing]},
  author={Lemercier, Jean-Marie and Richter, Julius and Welker, Simon and Moliner, Eloi and V{\"a}lim{\"a}ki, Vesa and Gerkmann, Timo},
  journal={IEEE Signal Processing Magazine},
  volume={41},
  number={6},
  pages={72--84},
  year={2025},
  publisher={IEEE}
}

@INPROCEEDINGS{solve_audio_inverse_10095637,
  author={Moliner, Eloi and Lehtinen, Jaakko and Välimäki, Vesa},
  booktitle={ICASSP 2023 - 2023 IEEE International Conference on Acoustics, Speech and Signal Processing (ICASSP)}, 
  title={Solving Audio Inverse Problems with a Diffusion Model}, 
  year={2023},
  volume={},
  number={},
  pages={1-5},
  doi={10.1109/ICASSP49357.2023.10095637}}

@inproceedings{evalNEURIPS2021_6e289439,
 author = {Kadkhodaie, Zahra and Simoncelli, Eero},
 booktitle = {Advances in Neural Information Processing Systems},
 editor = {M. Ranzato and A. Beygelzimer and Y. Dauphin and P.S. Liang and J. Wortman Vaughan},
 pages = {13242--13254},
 publisher = {Curran Associates, Inc.},
 title = {Stochastic Solutions for Linear Inverse Problems using the Prior Implicit in a Denoiser},
 url = {https://proceedings.neurips.cc/paper_files/paper/2021/file/6e28943943dbed3c7f82fc05f269947a-Paper.pdf},
 volume = {34},
 year = {2021}
}

@article{ML_uncertainty_H_llermeier_2021,
   title={Aleatoric and epistemic uncertainty in machine learning: an introduction to concepts and methods},
   volume={110},
   ISSN={1573-0565},
   howpublished={http://dx.doi.org/10.1007/s10994-021-05946-3},
   DOI={10.1007/s10994-021-05946-3},
   number={3},
   journal={Machine Learning},
   publisher={Springer Science and Business Media LLC},
   author={Hüllermeier, Eyke and Waegeman, Willem},
   year={2021},
   month=mar, pages={457–506} }

@article{lidc,
  title={The lung image database consortium {(LIDC)} and image database resource initiative {(IDRI)}: A completed
reference database of lung nodules on {CT} scans.},
  author={Armato, Samuel G and McLennan, Geoffrey and Bidaut, Luc and McNitt-Gray, Michael F and Meyer, Charles R},
  journal={Medical Physics},
  volume={38},
  pages={915-931},
 doi = {10.1148/radiol.2323032035},
  year={2011}
}

@inproceedings{EDM,
author = {Karras, Tero and Aittala, Miika and Laine, Samuli and Aila, Timo},
title = {Elucidating the design space of diffusion-based generative models},
year = {2022},
isbn = {9781713871088},
publisher = {Curran Associates Inc.},
address = {Red Hook, NY, USA},
booktitle = {Proceedings of the 36th International Conference on Neural Information Processing Systems},
articleno = {1926},
numpages = {13},
location = {New Orleans, LA, USA},
series = {NIPS '22}
}

@misc{Lung-PET-CT-Dx,
  author = {Li, P. and Wang, S. and Li, T. and Lu, J. and HuangFu, Y. and Wang, D.},
  title  = {A Large-Scale CT and PET/CT Dataset for Lung Cancer Diagnosis (Lung-PET-CT-Dx)},
  year   = {2020},
  publisher = {The Cancer Imaging Archive},
  doi    = {10.7937/TCIA.2020.NNC2-0461},
  url    = {https://doi.org/10.7937/TCIA.2020.NNC2-0461}
}

@article{chung2022scoremri,
  title={Score-based diffusion models for accelerated MRI},
  author={Chung, Hyungjin and Ye, Jong Chul},
  journal={Medical Image Analysis},
  pages={102479},
  year={2022},
  publisher={Elsevier}
}

@misc{jalal2021robustcompressedsensingmri,
      title={Robust Compressed Sensing MRI with Deep Generative Priors}, 
      author={Ajil Jalal and Marius Arvinte and Giannis Daras and Eric Price and Alexandros G. Dimakis and Jonathan I. Tamir},
      year={2021},
      eprint={2108.01368},
      archivePrefix={arXiv},
      primaryClass={cs.LG},
      howpublished={https://arxiv.org/abs/2108.01368}, 
}

@misc{zbontar2019fastmriopendatasetbenchmarks,
      title={fastMRI: An Open Dataset and Benchmarks for Accelerated MRI}, 
      author={Jure Zbontar and Florian Knoll and Anuroop Sriram and Tullie Murrell and Zhengnan Huang and Matthew J. Muckley and Aaron Defazio and Ruben Stern and Patricia Johnson and Mary Bruno and Marc Parente and Krzysztof J. Geras and Joe Katsnelson and Hersh Chandarana and Zizhao Zhang and Michal Drozdzal and Adriana Romero and Michael Rabbat and Pascal Vincent and Nafissa Yakubova and James Pinkerton and Duo Wang and Erich Owens and C. Lawrence Zitnick and Michael P. Recht and Daniel K. Sodickson and Yvonne W. Lui},
      year={2019},
      eprint={1811.08839},
      archivePrefix={arXiv},
      primaryClass={cs.CV},
      howpublished={https://arxiv.org/abs/1811.08839}, 
}

@article{wiesner2019cytopacqlinearinverse,
  title={CytoPacq: a web-interface for simulating multi-dimensional cell imaging},
  author={Wiesner, David and Svoboda, David and Ma{\v{s}}ka, Martin and Kozubek, Michal},
  journal={Bioinformatics},
  volume={35},
  number={21},
  pages={4531--4533},
  year={2019},
  publisher={Oxford University Press}
}

@article{BayesianMRILuo_2023,
   title={Bayesian MRI reconstruction with joint uncertainty estimation using diffusion models},
   volume={90},
   ISSN={1522-2594},
   howpublished={http://dx.doi.org/10.1002/mrm.29624},
   DOI={10.1002/mrm.29624},
   number={1},
   journal={Magnetic Resonance in Medicine},
   publisher={Wiley},
   author={Luo, Guanxiong and Blumenthal, Moritz and Heide, Martin and Uecker, Martin},
   year={2023},
   month=Mar, pages={295–311} }

@misc{coeurdoux2023plugandplaysplitgibbssampler,
      title={Plug-and-Play split Gibbs sampler: embedding deep generative priors in Bayesian inference}, 
      author={Florentin Coeurdoux and Nicolas Dobigeon and Pierre Chainais},
      year={2023},
      eprint={2304.11134},
      archivePrefix={arXiv},
      primaryClass={stat.ML},
      howpublished={https://arxiv.org/abs/2304.11134}, 
}

@misc{zach2025statisticalbenchmarkdiffusionposterior,
      title={A Statistical Benchmark for Diffusion Posterior Sampling Algorithms}, 
      author={Martin Zach and Youssef Haouchat and Michael Unser},
      year={2025},
      eprint={2509.12821},
      archivePrefix={arXiv},
      primaryClass={eess.SP},
      howpublished={https://arxiv.org/abs/2509.12821}, 
}

@book{kaipio2005statistical,
  title={Statistical and computational inverse problems},
  author={Kaipio, Jari P and Somersalo, Erkki},
  year={2005},
  publisher={Springer}
}

@article{stuart2010inverse,
  title={Inverse problems: a Bayesian perspective},
  author={Stuart, Andrew M},
  journal={Acta numerica},
  volume={19},
  pages={451--559},
  year={2010},
  publisher={Cambridge University Press}
}

@misc{KSDliu2016kernelizedsteindiscrepancygoodnessoffit,
      title={A Kernelized Stein Discrepancy for Goodness-of-fit Tests and Model Evaluation}, 
      author={Qiang Liu and Jason D. Lee and Michael I. Jordan},
      year={2016},
      eprint={1602.03253},
      archivePrefix={arXiv},
      primaryClass={stat.ML},
      howpublished={https://arxiv.org/abs/1602.03253}, 
}

@misc{gong2021slicedkernelizedsteindiscrepancy,
      title={Sliced Kernelized Stein Discrepancy}, 
      author={Wenbo Gong and Yingzhen Li and José Miguel Hernández-Lobato},
      year={2021},
      eprint={2006.16531},
      archivePrefix={arXiv},
      primaryClass={cs.LG},
      howpublished={https://arxiv.org/abs/2006.16531}, 
}

@misc{gorham2019measuringsamplequalitysteins,
      title={Measuring Sample Quality with Stein's Method}, 
      author={Jackson Gorham and Lester Mackey},
      year={2019},
      eprint={1506.03039},
      archivePrefix={arXiv},
      primaryClass={stat.ML},
      howpublished={https://arxiv.org/abs/1506.03039}, 
}

@misc{UQMRIedupuganti2020uncertaintyquantificationdeepmri,
      title={Uncertainty Quantification in Deep MRI Reconstruction}, 
      author={Vineet Edupuganti and Morteza Mardani and Shreyas Vasanawala and John Pauly},
      year={2020},
      eprint={1901.11228},
      archivePrefix={arXiv},
      primaryClass={cs.CV},
      howpublished={https://arxiv.org/abs/1901.11228}, 
}

@article{kumar2024multiinverseDesign,
  title={Multi-solution inverse design in photonics using generative modeling},
  author={Kumar, Preetam and Patra, Aniket and Shivaleela, ES and Caligiuri, Vincenzo and Krahne, Roman and De Luca, Antonio and Srinivas, T},
  journal={Journal of the Optical Society of America B},
  volume={41},
  number={2},
  pages={A152--A160},
  year={2024},
  publisher={Optica Publishing Group}
}

@inproceedings{gorham2017measuring,
  title={Measuring sample quality with kernels},
  author={Gorham, Jackson and Mackey, Lester},
  booktitle={International Conference on Machine Learning},
  pages={1292--1301},
  year={2017},
  organization={PMLR}
}

@article{FIDheusel2017gans,
  title={Gans trained by a two time-scale update rule converge to a local nash equilibrium},
  author={Heusel, Martin and Ramsauer, Hubert and Unterthiner, Thomas and Nessler, Bernhard and Hochreiter, Sepp},
  journal={Advances in neural information processing systems},
  volume={30},
  year={2017}
}

@inproceedings{LPIPSzhang2018unreasonable,
  title={The unreasonable effectiveness of deep features as a perceptual metric},
  author={Zhang, Richard and Isola, Phillip and Efros, Alexei A and Shechtman, Eli and Wang, Oliver},
  booktitle={Proceedings of the IEEE conference on computer vision and pattern recognition},
  pages={586--595},
  year={2018}
}
